%% file: main.tex
\documentclass{article}

\usepackage[preprint]{neurips_2026}

\usepackage[utf8]{inputenc}
\usepackage[T1]{fontenc}
\usepackage{microtype}
\usepackage{graphicx}
\usepackage{subcaption}
\usepackage{booktabs}
\usepackage{algorithm}
\usepackage{algorithmic}
\usepackage{hyperref}
\usepackage{url}
\usepackage{amsmath}
\usepackage{amssymb}
\usepackage{mathtools}
\usepackage{amsthm}
\usepackage{nicefrac}
\usepackage{xcolor}
\usepackage[capitalize,noabbrev]{cleveref}
\usepackage[disable,textsize=tiny]{todonotes}
\usepackage{enumitem}
\usepackage{amsfonts}
\usepackage{xspace}
\usepackage{arydshln}

\makeatletter
\@ifundefined{theHalgorithm}
  {}
  {}
\makeatother

\newcommand{\ours}{\textsc{QTPT}\xspace}

\theoremstyle{plain}
\newtheorem{theorem}{Theorem}[section]
\newtheorem{proposition}[theorem]{Proposition}
\newtheorem{lemma}[theorem]{Lemma}

\theoremstyle{definition}
\newtheorem{definition}[theorem]{Definition}
\newtheorem{assumption}[theorem]{Assumption}
\theoremstyle{remark}

\title{From Weak Data to Strong Policy: Q-Targets Enable Provable In-Context Reinforcement Learning%
\thanks{Accepted at NeurIPS 2026. This is an author preprint.}}

\author{%
  Yichen Lin$^{1}$ \quad Xuyuan Xiong$^{1}$ \quad Xue Wang$^{2}$ \quad Xiangfu Meng$^{1}$ \\
  Mike Mingcheng Wei$^{3}$ \quad Tao Yao$^{1}$\thanks{Corresponding author.} \\[0.5ex]
  {\normalfont$^{1}$Shanghai Jiao Tong University, Shanghai, China} \\
  {\normalfont$^{2}$Alibaba Group} \\
  {\normalfont$^{3}$School of Management, University at Buffalo, Buffalo, NY, USA} \\[0.5ex]
  {\normalfont$^{1}$\texttt{\{linyichen,xxy2021,mengxiangfu,taoyao\}@sjtu.edu.cn}} \\
  {\normalfont$^{2}$\texttt{wxie91@gmail.com}\quad $^{3}$\texttt{mcwei@buffalo.edu}}
}

\begin{document}

\maketitle

\begin{abstract}

\input{sections/abstract}
\end{abstract}

\input{sections/intro}
\input{sections/model_setting}
\input{sections/analysis}
\input{sections/experiment}

\input{sections/conclusion}

\section*{Broader Impact}
This paper presents work whose goal is to advance the field of Machine Learning.
There are many potential societal consequences of our work, none which we feel
must be specifically highlighted here.

\bibliographystyle{plainnat}
\bibliography{references,rebuttal_references}

\appendix
\input{sections/appendix}

\newpage
\input{checklist}

\end{document}

%% file: sections/abstract.tex
Existing in-context reinforcement learning methods mainly pretrain Transformers with supervised behavior-prediction objectives. This enables task inference from context, but makes the learned policy strongly depend on the quality of offline actions: when trajectories are weak or suboptimal, imitation itself becomes a biased learning signal. We propose Q-Target Pretrained Transformers (QTPT), which keeps the context-conditioned Transformer architecture but replaces behavior cloning with a Bellman-style Q-target objective. QTPT therefore learns to use rewards and transitions in the context to estimate action values, rather than simply imitating the behavior policy. We theoretically analyze QTPT in stochastic linear bandits and finite-horizon MDPs, showing stronger robustness to data quality than supervised pretraining. Empirically, QTPT improves over supervised behavior prediction on controlled RL benchmarks with random or suboptimal data, and we examine extensions to D4RL Kitchen and AntMaze. Supplementary experiments evaluate backbone robustness, meta-RL comparisons, task-coherent context, and unsupported-action value overestimation. These comparisons distinguish the benefits of Q-target pretraining from the remaining limitations of offline coverage.


%% file: sections/intro.tex
\vspace{-0.5cm}
\section{Introduction}
\vspace{-0.1cm}

%
Recent advances in large language models have demonstrated their remarkable ability to perform various tasks in a zero-shot or few-shot manner using in-context learning \citep{brown2020languagemodelsfewshotlearners,garg2023transformerslearnincontextcase}. Building upon this paradigm, transformer-based models have been successfully applied to Reinforcement Learning (RL) settings, showing strong In-Context Reinforcement Learning (ICRL) capabilities \citep{laskin2022context,lee2023supervised,lin2024transformersdecisionmakersprovable}. These models can implicitly capture temporal dependencies within sequences of state-action-reward tuples, enabling them to generalize and make decisions in unseen environments.

In-context RL requires inferring the current task from interaction history and selecting actions using that information. Supervised approaches include Algorithm Distillation \citep{laskin2022context}, which distills learning trajectories, and Decision Pretrained Transformer \citep{lee2023supervised}, which uses optimal-action supervision. These supervision sources differ. Our matched SPT baseline predicts the behavior actions recorded in the offline data; when those actions are weak or suboptimal, its targets encourage imitation of weak decisions.

To address this objective-level limitation, we propose Q-Target Pretrained Transformers (\ours), a pretraining framework that replaces supervised behavior cloning with a Bellman-style Q-learning objective. The context still provides the information needed for task adaptation, but the learning signal comes from rewards, transitions, and bootstrapped value estimates rather than behavior actions alone. As a result, \ours learns a context-conditioned Q-function from offline datasets, even when these datasets are generated by suboptimal policies (see Figure \ref{fig:qtpt}).
By aligning Transformer pretraining with value-based reinforcement learning, \ours reduces the reliance on expert trajectories or optimal action labels and improves robustness under weak offline data.

\begin{figure*}[t]
    \centering
    \includegraphics[width=0.6\textwidth]{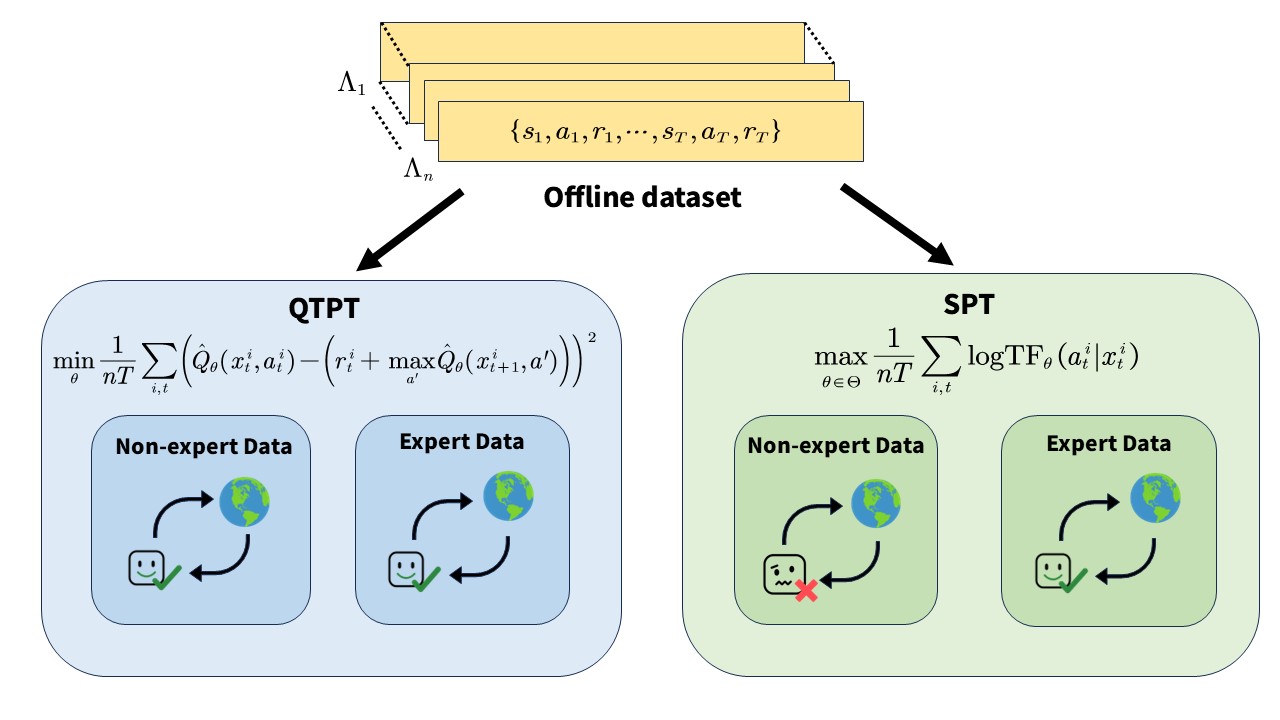}
    \caption{Comparison of Q-target Pretrained Transformers (QTPT) and Supervised Pretrained Transformers (SPT) for in-context RL. While SPT requires expert demonstrations for good performance, QTPT maintains robust performance even when trained on suboptimal data.}
    \label{fig:qtpt}
\end{figure*}
Theoretically, we give a Transformer approximation construction and a suboptimality upper bound with two distinct components.
The first component arises from the sample bias induced by behavior policies, and the second component stems from the model bias affected by the Rademacher complexity and the size of the training dataset.
Such a clear separation builds a novel theoretical foundation to isolate the impacts of sample bias and model bias and sheds light on how \ours balances leveraging available data and mitigating model bias during pretraining. 

We evaluate \ours on controlled RL benchmarks, D4RL Kitchen and AntMaze, and a language-model reasoning adaptation. Matched comparisons with supervised pretraining test the effect of replacing behavior labels with Q-targets, and separate ablations examine context and the backup target. Supplementary experiments compare meta-RL baselines, decoder backbones, and explicit conservative regularization; each comparison retains its own data and evaluation protocol rather than being pooled into one performance ranking.

The major contributions of this work are summarized as follows:

1. We propose QTPT, a context-conditioned Q-target pretraining framework for in-context RL. Unlike supervised pretraining, QTPT does not treat behavior actions as ground-truth labels; it trains the Transformer to estimate action values from contextual rewards and transitions.
By integrating Q-learning with Transformers, we enable end-to-end value estimation inside a sequence model.
This shift from behavior prediction to Q-target pretraining is the central algorithmic contribution of the paper.

2. We derive suboptimality upper bounds for stochastic linear bandits and finite-horizon MDPs, separating sample bias from model bias. The general-MDP analysis contains a worst-case $\widetilde{\mathcal O}(\sqrt L\,T^2n^{-1/4})$ statistical term, together with class-complexity and approximation terms. The bandit model-bias component has $\widetilde{\mathcal O}(\sqrt{T/n})$ dependence under its stated conditions; it is not the entire combined bound.

3. We evaluate the roles of supervision, context, and backup construction. The experiments show strong gains over behavior prediction on random data, test the pattern across three decoder families, and directly measure unsupported-action value overestimation. Additional experiments probe meta-RL comparisons and the effect of task-coherent versus mixed-task context.


\subsection{Related Work}

\textbf{Q-Learning}.
Q-learning is a foundational off-policy reinforcement learning algorithm that updates Q-values using the Bellman equation to learn optimal policies through environmental interaction \citep{Watkins1992qlearning,sutton2018reinforcement}. Over the years, a range of advancements have improved Q-learning's efficiency and practical applicability—especially regarding exploration-exploitation trade-offs. Upper-Confidence Bounds (UCB) have enabled more efficient exploration, achieving regret bounds on par with model-based methods \citep{jin2018qlearningprovablyefficient,zanette2019tighter}. Optimistic Q-learning algorithms \citep{even2001convergence,yang2021qlearninglogarithmicregret} further reduce regret under appropriate conditions, marking a notable improvement over earlier approaches that suffered from poor sample complexity. Deep Q-Networks (DQN)~\citep{mnih2015human} integrates the Q-learning with deep learning to achieve human-level performance in complex tasks. \cite{kapturowski2018recurrent} leverage LSTM networks to better capture temporal dependencies in RL tasks.

Recent advances have adapted Q-learning to large-scale sequence models. Q-SFT \citep{hong2024qsftqlearninglanguagemodels} reframes Q-learning as a supervised fine-tuning problem, optimizing a weighted cross-entropy loss to approximate Q-values and providing theoretical analysis of convergence. Similarly, the Q-learning Decision Transformer~\citep{yamagata2023q} employs a two-step approach: it first relabels reward-to-go (RTG) values using Q-learning, then trains a Decision Transformer~\citep{chen2021decision} on these relabeled trajectories.

QTPT instead trains a history-conditioned Q-function directly with sampled Bellman targets. The backup is constructed by the training algorithm; it is not an architectural claim that every attention layer performs a Bellman update. The Transformer representation supports inference from interaction history, while the objective uses rewards and transitions rather than behavior-action labels. The appendix construction concerns representational capacity and does not guarantee that gradient-based training discovers that construction.


\textbf{Offline Reinforcement Learning}.
Offline Reinforcement Learning (RL) focuses on learning optimal policies from pre-collected datasets, without additional interactions with the environment \citep{levine2020offlinereinforcementlearningtutorial,matsushima2020deploymentefficientreinforcementlearningmodelbased}. 
One of the major challenges in offline RL is the distribution drift between the behavior policy, which generates the dataset, and the learned policy \citep{levine2020offlinereinforcementlearningtutorial,kostrikov2021offlinereinforcementlearningimplicit,rashidinejad2021bridging}. 
Techniques such as conservative value function estimation, policy constraints, and regularization have been developed to mitigate this challenge~\citep{wu2019behaviorregularizedofflinereinforcement,kumar2020conservativeqlearningofflinereinforcement,kidambi2021morelmodelbasedoffline,jin2022pessimismprovablyefficientoffline,dong2023pastapessimisticassortmentoptimization,hu2024qvalueregularizedtransformeroffline,qu2024recursiveintrospectionteachinglanguage,setlur2024rlincorrectsyntheticdata,park2024valuelearningreallymain}. 
QTPT is conceptually related to Transformer-based ICRL methods such as AD~\citep{laskin2022context} and DPT~\citep{lee2023supervised}, but targets a different question: rather than studying whether supervised sequence models can infer tasks from high-quality trajectories, we ask how changing the pretraining objective affects robustness when the offline behavior data are weak or suboptimal.
The advantages of offline reinforcement learning over behavior cloning (imitation learning) are studied in \cite{kumar2022preferofflinereinforcementlearning}. Unlike this stream of work, we utilize \emph{meta-learning framework} to generalize and solve \emph{unseen} RL tasks after pretraining, while offline RL generally focuses on solving \emph{same} RL tasks from which the offline datasets were originally collected.

QTPT combines Bellman-based value learning with a context-conditioned sequence model evaluated without test-time parameter updates. CQL adds conservative value regularization, whereas IQL uses an in-sample value estimate in its critic backup and advantage-weighted policy extraction; IQL is not Bellman-free. Q-SFT and QDT use value information to supervise sequence-model policies. Our comparisons distinguish these learning mechanisms from the models' input and deployment protocols (Appendix~\ref{app:baseline_compare}).

\textbf{In-Context Learning and Reinforcement Learning}.
Without test-time parameter updates, in-context learning (ICL) adapts its predictions or decisions to a new task by conditioning on context observed at inference time \citep{brown2020languagemodelsfewshotlearners,min-etal-2022-metaicl,garg2023transformerslearnincontextcase}.
Mechanistic explanations of ICL include implicit statistical inference and gradient-based learning implemented by attention \citep{akyurek2022learning,bai2023transformers,von2023transformers,zhang2023trained}.
In-Context Reinforcement Learning (ICRL) addresses sequential decision-making problems by inferring optimal policies from past trajectories. In the realm of ICRL,~\cite{wang2025transformerslearntemporaldifference}  explores how transformers can implement Temporal Difference (TD) learning directly in the forward pass. See~\cite{moeini2025surveyincontextreinforcementlearning} for a more comprehensive survey.  While existing literature mainly focus on the ICRL with supervised pretraining paradigm (e.g.,~\cite{lin2024transformersdecisionmakersprovable}),
our approach leverages Q-learning as the pretraining strategy, explicitly minimizing in-context temporal difference errors. This approach directly aligns the pretraining objective with subsequent in-context decision-making tasks, especially when noisy or suboptimal data are involved.

\paragraph{Meta-RL and belief-state control.}
QTPT shares the goal of adaptation across tasks with RL$^2$ \citep{rev_duan2016rl2} and PEARL \citep{rev_rakelly2019pearl}. RL$^2$ meta-trains recurrent policies with online interaction; PEARL combines inferred task context with replay-based actor--critic learning; QTPT uses a fixed offline trajectory corpus. These are related adaptation problems with different data-acquisition protocols, not disjoint definitions of ICRL. In a belief-MDP or Bayes-adaptive MDP, history determines a belief over latent states or task parameters \citep{rev_kaelbling1998pomdp,rev_duff2002bamdp}. QTPT uses history as input to an approximate value function without explicitly computing that posterior. Its contribution is the offline Q-target pretraining objective and its analysis, not the introduction of history as a belief proxy.

%% file: sections/model_setting.tex
\vspace{-0.2cm}
\section{Model Setting}
\vspace{-0.1cm}
\subsection{Preliminaries}

We consider a set of decision-making environments $\mathcal{M}$, each operating over $T$ rounds with shared state and action spaces $(\mathcal{S}, \mathcal{A})$. Each environment $M \in \mathcal{M}$ has a unique transition model $\mathbb{P}_M: \mathcal{S} \times \mathcal{A} \rightarrow \Delta(\mathcal{S})$, initial state distribution $\eta_M\in \Delta(\mathcal{S})$, and a reward function $r_M: \mathcal{S}\times \mathcal{A} \to \Delta(\mathbb{R}) $. The agent's uncertainty is captured by environment priors $\Lambda_{\text{train}} \in \Delta(\mathcal{M})$ and $\Lambda_{\text{test}} \in \Delta(\mathcal{M})$ for training and testing respectively.
This framework represents various scenarios, including  \(T\) rounds of multi-armed bandit problems and \(K\) episodes of \(H\)-step MDPs with \(T = KH\). The training stage uses an offline dataset $D$ consisting $n$ offline trajectories $\{D_T^i = (s_1^i, a_1^i, r_1^i, \dots, s_T^i, a_T^i, r_T^i)\}_{i=1}^n$, which have been collected by the behavior policy $\pi_\beta$. 

We denote a partial interaction trajectory—comprising the sequence of observed states, actions, and rewards—by \(D_t=\{s_1, a_1, r_1, \cdots, s_t, a_t, r_t\}\in\mathcal{T}_t=(\mathcal{S}\times \mathcal{A}\times\mathbb{R})^t\). 
For convenience, we define the \emph{context state} as \(x_t = D_{t-1} \cup \{s_t\}\), which lies in the space \(\mathcal{X}_t = \mathcal{T}_{t-1} \times \mathcal{S}\). This representation captures the cumulative interaction history up to round $t-1$ and the current state $s_t$.  
Analogous to standard state transitions in reinforcement learning, the context state evolves based on the current context and action. Specifically, the next context state is given by: 
$x_{t+1} = x_t \cup \{a_t, r_t, s_{t+1}\}$, $r_t\sim r_M(\cdot|s_t, a_t), s_{t+1} \sim \mathbb{P}_M(\cdot|s_t, a_t)$. 

We formalize this dynamic with the  \emph{context transition model}, defined as \(\mathbb{T}_{M,t} : \mathcal{X}_t \times \mathcal{A} \to \Delta(\mathcal{X}_{t+1})\), which maps a given context-action pair to a distribution over the next context state.
A \emph{policy} \(\pi\) maps a context state \(x_t \in \mathcal{X}_t\) to a distribution over the actions \(\pi(\cdot|x_t) \in \Delta(\mathcal{A})\).  

For the finite-horizon analysis we use the undiscounted setting, as in \citet{lin2024transformersdecisionmakersprovable}; implementation-specific discounting is stated separately. We define the value function \(V_{M,t}^{\pi} : \mathcal{X}_t \rightarrow \mathbb{R}\), the Q-function \(Q_{M,t}^{\pi} : \mathcal{X}_t \times \mathcal{A} \rightarrow \mathbb{R}\) for policy $\pi$ , and the Bellman operator \(\Gamma: \mathbb{R}^{\mathcal{X}\times \mathcal{A}}\rightarrow \mathbb{R}^{\mathcal{X}\times \mathcal{A}}\)  as follows:
\begin{align}
    & V_{M,t}^{\pi}(x_t)  = \mathbb{E} \left[ \sum_{h=t}^{T} r_h \,\middle|\, x_t \right] \label{eq:value}\\
    & Q_{M,t}^{\pi}(x_t,a_t) \\
    =& \mathbb{E} \left[ r_M(s_t, a_t) \right] + \mathbb{E}_{x_{t+1} \sim \mathbb{T}_{M,t}(\cdot|x_t,a_t)} \left[ V_{M,t+1}^{\pi}(x_{t+1}) \right] \label{eq:qvalue} \\
    & (\Gamma Q_{M,t})(x_t,a_t) \\
    =& \mathbb{E}[r_M(s_t,a_t)] + \mathbb{E}_{x_{t+1} \sim \mathbb{T}_{M,t}(\cdot|x_t,a_t)} \left[\max_{a'} Q_{M,t+1}(x_{t+1},a')\right] \label{eq:bellman}
\end{align}
 
Finally, the optimal value function and Q-function are given by $V_{M,t}^*(x_t) = \max_{\pi} V_{M,t}^{\pi}(x_t)$ and $Q_{M,t}^*(x_t,a_t) = \max_{\pi} Q_{M,t}^{\pi}(x_t,a_t)$ for $\forall x_t \in \mathcal{X}_t,\; a_t \in \mathcal{A}$ respectively.





\begin{algorithm}[tbp]
\caption{Q-target Pretrained Transformers (\ours)}
\label{algo:Q}
\begin{algorithmic}[1]
\STATE \textbf{Input}: an Offline Dataset $D$ collected in $\Lambda_{\text{train}}$, and the horizon $T$. 
\STATE \textcolor{blue}{{// Pretraining model on dataset}}
\STATE Randomly initialize $\hat{Q}_\theta$
\WHILE{not converged}
   \STATE Sample a batch of $\mathcal{B}$ trajectories from $D$
    \STATE Compute the Temporal Difference (TD) target $y_{t}^i = r_t^i + \max_{a' \in \mathcal{A}} \hat{Q}_{\theta,t+1}(s_{t+1}^i, a', D^i_{t})$
 \STATE Compute the loss: 
$\mathcal{L}_{\mathcal{B}}(\hat{Q}_{\theta}) = \frac{1}{|\mathcal{B}|T}\sum_{i=1}^{|\mathcal{B}|}\sum_{t=1}^T{\left( \hat{Q}_{\theta,t}(s_t^i, a_t^i, D^i_{t-1}) - y^i_{t} \right)^2}$
   \STATE Backpropagate to update $\theta$
\ENDWHILE
\STATE \textcolor{blue}{{// Online test-time deployment}}
\STATE Initialize an empty dataset $D$ and sample environment $M \sim \Lambda_{\text{test}}$
\FOR{$t$ in horizon $T$}
   \STATE Choose $a_t \in \arg\max_{a} \hat{Q}_\theta(s_t, a, D)$
   \STATE Execute $a_t$ and observe $r_t, s_{t+1}$
   \STATE Update $D$ with $(s_t, a_t, r_t)$
\ENDFOR
\end{algorithmic}

\end{algorithm}
\vspace{-1ex}
\subsection{Algorithm}\label{sec:algorithm}
\vspace{-1ex}
This section presents the Q-target Pretrained Transformer (QTPT) algorithm. The model is a parameter-shared causal decoder: $(\hat Q_1,\ldots,\hat Q_T)$ denotes time-indexed outputs, not $T$ independently stored networks. History and current-state inputs are embedded, and an action-value readout produces $\hat Q_t(x_t,a;\theta)$. All training tasks update the same parameters; the context supplied to an individual prediction remains task-coherent. 
To learn from the  offline dataset $D = \{(s_1^i, a_1^i, r_1^i, \dots, s_T^i, a_T^i, r_T^i)\}_{i=1}^n$, we utilize a class of sequence 
functions $\hat{\mathcal{Q}}=(\hat{\mathcal{Q}}_1 \times \cdots \times \hat{\mathcal{Q}}_T)$, 
where each $\hat{Q}_t \in \hat{\mathcal{Q}}_t$ tries to approximate the optimal Q-value function at time step $t$. For instance, a Transformer can be viewed as a model of sequential $Q$ functions. For simplicity, we denote the full Q-function approximator as $\hat{Q} = (\hat{Q}_1, \dots, \hat{Q}_T) \in \hat{\mathcal{Q}}$ and optimize it as follows:
\vspace{-0.2cm}
\begin{align}\label{learn_obj_x}
   \underset{\theta\in \Theta}{\min} \quad \frac{1}{nT} \sum_{i=1}^{n} \sum_{t=1}^T \Big( \hat{Q}_t(x_t^i, a_t^i; \theta) \\
   - \big( r_t^i + \max_{a'} \hat{Q}_{t+1}(x_{t+1}^i, a'; \theta) \big) \Big)^2,
\end{align}

where $n$ is the number of trajectories from $D$ and we parameterize the approximated Q-value function is by $\theta$.
The formal algorithm of \emph{Q-Target Pretrained Transformer} (\ours) is stated in Algorithm \ref{algo:Q}. During the pretraining stage, the approximated Q-value functions are learned in a way that minimizes the differences between the predicted Q-values and the bootstrapped targets. At test time, actions are selected greedily while $\theta$ remains fixed and $x_t=D_{t-1}\cup\{s_t\}$ grows. We set $Q_{T+1}=0$ at the terminal step. Algorithm~\ref{algo:Q} and Eq.~\eqref{learn_obj_x} display the idealized hard-max objective; practical targets are detached from the gradient. The submitted implementation stabilizes this objective with a periodically synchronized target network and a softmax-weighted backup, specified in Appendix~\ref{app:target_details}. These mechanisms are not an explicit pessimistic objective.

\paragraph{A batch is not one context.}
Sampling trajectories from many tasks trains a shared history-to-value mapping, but does not make all those trajectories evidence about the task of a particular query. For example, two latent tasks can have identical observations but opposite rewarding actions. Their unlabeled pooled histories do not identify which task is active, whereas rewards observed within the active trial do. Additional same-task episodes can help when their identity is known and they are also available at test time. The setup permits a task-coherent trial of $K$ episodes, $T=KH$; it does not authorize adding unrelated tasks to that trial's context. Appendix~\ref{app:context_diagnostics} tests this distinction.

%


\vspace{-1ex}
\subsection{Analysis Framework}
\vspace{-1ex}

Each approximated function \(\hat{Q}_t\) induces a time-dependent greedy policy \(\pi_{\hat{Q}_t}\). 
The overall policy \(\pi_{\hat{Q}}\) is defined as the sequence of time-dependent policies \(\{\pi_{\hat{Q}_t}\}_{t=1}^T\), greedily selecting actions at each time step. 
Let $\mathcal{L}_n(Q)=\mathbb{E}_{D^n_T\sim\pi_{\beta}}[(Q(x,a)-(\Gamma Q)(x,a))^2]$ denote the Bellman-error objective induced by the offline dataset. We use $\hat{Q}_n^*=\arg\min_Q\mathcal{L}_n(Q)$ for the ideal Bellman-error minimizer over the reference value-function class, and $\tilde{Q}_n=\arg\min_{Q\in \hat{\mathcal{Q}}}\mathcal{L}_n(Q)$ for the best solution realizable by the Transformer hypothesis class $\hat{\mathcal{Q}}$. In practice, QTPT trains a parameterized Transformer to approximate this latter object; the analysis below separates the statistical effect of the offline dataset from the approximation/modeling effect of restricting to $\hat{\mathcal{Q}}$.
We use $\epsilon_{\mathrm{tf}}\geq0$ to denote the irreducible approximation floor induced by restricting the Bellman-error minimization to $\hat{\mathcal{Q}}$. It is zero under Transformer realizability, but need not vanish for a fixed misspecified Transformer class.
The approximation result is a capacity/existence statement, not a guarantee of parameter recovery by the optimizer. Consistency requires the approximation error to vanish, for example under realizability or an appropriate growing-capacity sequence; a fixed nonzero floor cannot be hidden in a vanishing rate.
The induced policy $\pi_{\tilde{Q}_n}$ greedily selects actions per \(\{\pi_{\tilde{Q}_{n,t}}\}_{t=1}^{T}\) (omitting $\theta$ for readability). 
We define the expected cumulative reward for policy $\pi$ is $\mathbb{E}_{M\sim\Lambda,x_1\sim \eta_M}\big[ V_{M,1}^\pi(x_1)\big]$, and the optimal expected cumulative reward is $\mathbb{E}_{M\sim\Lambda,x_1\sim \eta_M}\big[ V_{M,1}^*(x_1)\big]$.
The suboptimality gap quantifies the expected cumulative reward difference between the optimal policy \(\pi^*\) and the learned policy \(\pi_{\tilde{Q}_n}\). This gap can be decomposed into two components: sample bias and model bias. Specifically: 
\begin{equation}\label{eq_rev_dif}
\begin{array}{rl}
\mathsf{Subopt}_{\Lambda}(\pi^*,\pi_{\tilde{Q}_n})
&= \underbrace{\mathsf{Subopt}_{\Lambda}(\pi^*,\pi_{\hat{Q}_n^*})}_{\text{Sample Bias}} \\
&+ 
\underbrace{\mathsf{Subopt}_{\Lambda}(\pi_{\hat{Q}_n^*},\pi_{\tilde{Q}_n})}_{\text{Model Bias}},
\end{array}
\end{equation}
Sample bias arises from the finite offline dataset and behavior policy, which may not fully represent the true environment dynamics.
Model bias stems from the approximation and optimization limits of the Transformer hypothesis class relative to the ideal Bellman-error minimizer.

%% file: sections/analysis.tex
\vspace{-0.3cm}
\section{Main Analysis and Results}
\vspace{-0.1cm}
To streamline the analysis of Eq. \ref{eq_rev_dif}, we denote $d^{\pi^*}_{M,t}(x_t,a_t)$ as the marginal distribution at time $t$ for the optimal policy in the environment $M$,  
and $d^{\pi_{\beta}}_{M,t}(x_t,a_t)$ for any behavior policy $\pi_{\beta}$.
Define the marginal occupancy as $d^{\pi}_{M,t}(x)=\sum_a d^{\pi}_{M,t}(x,a)$.  
The analysis follows Eq.~\eqref{eq_rev_dif}: Proposition~\ref{prop_sb} controls sample bias; Propositions~\ref{linearbandit} and \ref{prop_mb} control model bias; Proposition~\ref{prop_3} combines them. The coverage parameter $L$ is the reciprocal of the positive-occupancy lower bound below. Larger $L$ weakens the guarantee through $\sqrt L$, independently of whether behavior actions have high reward. We state the required assumptions explicitly:




\begin{assumption}[Bounded Rewards and Distributional Coverage]\label{assump:bound}
    We assume the reward function $r_M$ and the marginal state-action distributions $d^{\pi^*}_{M,t}$ and $d^{\pi_\beta}_{M,t}$ satisfy the following conditions:
    \begin{enumerate}[leftmargin=0.2in, nolistsep, label=(\alph*)]
    \item \textbf{(Bounded Rewards)} The reward function is uniformly bounded: $|r_M(s,a)| \le 1$ for all $(s,a) \in \mathcal{S} \times \mathcal{A}$.
    
    \item \textbf{(Optimal Policy Concentrability)} The behavior policy $\pi_\beta$ provides sufficient support for the optimal policy $\pi^*$: if $d^{\pi^*}_{M,t}(x_t,a_t) > 0$, then $d^{\pi_\beta}_{M,t}(x_t,a_t) > 0$, for all $t \in [T]$.
    
    \item \textbf{(Lower-Bounded Occupancy)} The occupancy density under $\pi_\beta$ is bounded on its support: 
    \[
    L^{-1} := \inf_{(x_t,a_t):\, d^{\pi_\beta}_{M,t}(x_t,a_t) > 0} d^{\pi_\beta}_{M,t}(x_t,a_t), \quad \forall t \in [T].
    \]
    \end{enumerate}
\end{assumption}
\vspace{-0.2cm}
    
The bounded-reward assumption controls the remaining-return range.  Optimal Policy Concentrability assumes that $d^{\pi_{\beta}}$ covers the trajectory of some optimal policy $\pi^*$, and Lower-Bounded Occupancy ensures that the positive occupancy density under $\pi_{\beta}$ is bounded away from $0$ on its support. The infimum excludes zero-occupancy context--action pairs; condition (b) separately requires support for an optimal policy. Neither statement guarantees a small $L$ in a large history space. Note that these two assumptions (see Assumptions 4.1 and 4.2 in \cite{nguyentang2023instancedependentboundsofflinereinforcement}) ensure that an optimal policy is statistically learnable from offline data and do not require the all-policy uniform coverage considered in the batch reinforcement learning literature (see Assumption 1 in \cite{duan2021risk} and \cite{chen2019information}). 


\vspace{-1ex}
\subsection{Bound of Sample Bias}
\vspace{-1ex}

Denote the function bound of $Q\in\hat{\mathcal{Q}}$ as $B_{Q}=\underset{x,a}{\sup}\|Q(x,a)-(\Gamma Q)(x,a)\|_2$. 
We will use the general notation $\mathcal{L}_n$ to denote the empirical loss calculated on $n$ samples, which corresponds to the $\mathcal{L_B}$ in our Algorithm on a mini-batch $\mathcal{B}$. 
We can now state the sample bias bound for \ours. 


\begin{proposition}\label{prop_sb}(\textbf{Bound of Sample bias in \ours})
If for all tuples $(D^i_{t-1},s^i_t,a^i_t,y^i_t)$, there exists a constant $B>0$ such that $B_{\hat{Q}^*_n}\leq B$.  Per Assumption \ref{assump:bound}, with the probability at least $1-\delta$, we have 

\begin{equation*}
\begin{array}{rl}
    \mathsf{Subopt}_{\Lambda}(\pi^*,\pi_{\hat{Q}^*_n}) &\leq  \mathcal{O}\Big(\sqrt{L} T^2\Big[(\frac{2\log (2T/\delta)}{n})^{1/4} \\
    &+\frac{c(\hat{\mathcal{Q}},n)}{\sqrt{n}}\Big]+\sqrt{L}T B\Big),
\end{array}
\end{equation*}

where $c(\hat{\mathcal{Q}},n) = \sqrt{\max\{\sqrt{n\log |\hat{\mathcal{Q}}|} ,\log |\hat{\mathcal{Q}}|\}}$.
\end{proposition}

The explicit $T^2$ dependence is a conservative worst-case upper bound, not a proven lower bound for sequence-history methods. Range-based concentration for squared Bellman error contributes $\widetilde{\mathcal O}(T^2n^{-1/2})$; taking a square root and accumulating across $T$ stages produces $\widetilde{\mathcal O}(\sqrt L\,T^2n^{-1/4})$. The closest full-history ICRL analysis also has a $T^2$ pretraining term \citep[Theorems~6 and~13]{lin2024transformersdecisionmakersprovable}. More favorable offline-RL rates have been established with additional structure and different algorithmic machinery: OPDVR exploits stationary tabular transitions and double variance reduction \citep[Theorems~3.2 and~4.1]{rev_yin2021opdvr}. These are not interchangeable bounds. Improving our exponent without changing the present assumptions remains open.
The term $\sqrt L\,TB$ retains Bellman approximation error. It does not vanish merely because the sample size increases; controlling it requires an appropriate value-function class. Similarly, the class-complexity factor is part of the guarantee, so a polynomial explicit horizon factor alone does not establish polynomial complexity for an unrestricted history representation.

\vspace{-1ex}
\subsection{Bound of Model Bias}
\vspace{-1ex}

We start with the simple stochastic linear bandit problem. At each time step $t= 1,2,...,T$, the agent selects an action $a_t\in \mathbb{R}^d$ from a set of actions $\mathcal{A}_t$. Upon taking action $a_t$, the agent receives a reward $r_t = \langle a_t, w^* \rangle + \epsilon_t$, where $w^*$ is an unknown parameter vector, and $\epsilon_t$ represents i.i.d. noise with bounded variance. Without loss of generality, $\|a_t\|_2\leq 1$. We define the Gram matrix of actions as $G_{t} = \sum_{i=1}^ta_ia_i^{\top}$, and show the upper bound for the $TF_{\theta}$ solved via Eq. \ref{learn_obj_x} in Proposition~\ref{linearbandit}.
\begin{proposition}(\textbf{Bound of Model Bias in \ours on Stochastic Linear Bandit})\label{linearbandit}
If the minimum eigenvalue of the Gram matrix $G_{t}$ satisfies $\lambda_{\text{min}}(G_t)\geq \alpha nt$ for some constant $\alpha>0$, then for the Transformer $TF_{\theta}(\cdot)$ at appropriate scale, with probability at least $1-\delta$, it satisfies 
\[
\mathsf{Subopt}(\pi_{\hat{Q}^*_n},\pi_{\tilde{Q}_n})
\leq \epsilon_{\mathrm{tf}}+\mathcal{O}\left(\sigma\sqrt{T}\sqrt{\frac{d\log T+\log(1/\delta)}{\alpha n}}\right).
\]
\end{proposition}

Proposition~\ref{linearbandit} bounds the bandit model-bias component under the stated Gram-matrix condition. Its finite-sample term decreases with $n$ and depends on dimension, horizon, noise, and coverage. For a fixed misspecified Transformer, $\epsilon_{\mathrm{tf}}$ remains. This is not an unconditional consistency or parameter-recovery claim.






For finite-horizon MDPs, we establish the upper bound on the model bias term in Proposition~\ref{prop_mb}. 
\begin{proposition}\label{prop_mb}(\textbf{Bound of Model Bias in \ours on Finite-Horizon MDP})
    With probability at least $1-\delta$, we have 
\vspace{-0.2cm}
\begin{align*}
    &\mathsf{Subopt}_{\Lambda}(\pi_{\hat{Q}^*_n},\pi_{\tilde{Q}_n})
\leq 2T\sqrt{\epsilon_{\mathrm{tf}}}+\mathcal{O}\Big( T^2 
\Big(\frac{1}{\sqrt{n}}
+\left(\frac{2\log(2T/\delta)}{n}\right)^{1/4}
\Big)\Big).
\end{align*}
\end{proposition}
Proposition~\ref{prop_mb} separates an $n^{-1/2}$ complexity term from an $n^{-1/4}$ confidence term, retaining the approximation contribution $2T\sqrt{\epsilon_{\mathrm{tf}}}$. The logarithmic confidence dependence is displayed explicitly; statements suppressing it use fixed confidence. A faster excess-risk estimate would improve the statistical rate, but is not assumed by this proposition.
\vspace{-1ex}
\vspace{-1ex}
\subsection{Final Bound}

Combining the sample bias bound (Proposition \ref{prop_sb}) and model bias bounds (Propositions \ref{linearbandit}, \ref{prop_mb}), we obtain the following final bound of \ours.
\vspace{-1ex}
\begin{proposition}\label{prop_3}(\textbf{Upper Bound of \ours})
Under Assumption~\ref{assump:bound} and the applicable model-bias proposition, assume $\Lambda_{\mathrm{test}}\ll\Lambda_{\mathrm{train}}$ and $\mathrm d\Lambda_{\mathrm{test}}/\mathrm d\Lambda_{\mathrm{train}}\leq\mathcal C$. Let
\[
S_n(\eta)=\mathcal O\!\left(\sqrt L\,T^2\left[\left(\frac{2\log(2T/\eta)}{n}\right)^{1/4}
+\frac{c(\hat{\mathcal Q},n)}{\sqrt n}\right]+\sqrt L\,TB\right).
\]
With probability at least $1-\delta$, the respective combined bounds are
\[
\begin{aligned}
\text{Linear bandit:}\quad
&\mathsf{Subopt}_{\Lambda_{\mathrm{test}}}(\pi^*,\pi_{\tilde Q_n})\\
&\leq\mathcal C\left[S_n(\delta/2)+\epsilon_{\mathrm{tf}}
+\mathcal O\!\left(\sigma\sqrt T\sqrt{\frac{d\log T+\log(2/\delta)}{\alpha n}}\right)\right],\\
\text{Finite-horizon MDP:}\quad
&\mathsf{Subopt}_{\Lambda_{\mathrm{test}}}(\pi^*,\pi_{\tilde Q_n})\\
&\leq\mathcal C\left[S_n(\delta/2)+2T\sqrt{\epsilon_{\mathrm{tf}}}
+\mathcal O\!\left(T^2\left[n^{-1/2}+
\left(\frac{2\log(4T/\delta)}{n}\right)^{1/4}\right]\right)\right].
\end{aligned}
\]
For each case, the sample-bias event and the applicable model-bias event receive failure probability $\delta/2$ each. A union bound gives overall failure probability at most $\delta$; independence is not needed.
\end{proposition}
\vspace{-0.2cm}
The constant $\mathcal{C}$ defined in Proposition \ref{prop_3} describes the environment-level distribution shift between the set of environments used during pretraining $\Lambda_{\text{train}}$ and those encountered during testing $\Lambda_{\text{test}}$, which can differ in aspects like the state transitions, reward structures and the size of action space, for example.
It reflects how well the pretraining data prepares the model for test environments. 

Proposition \ref{prop_3} establishes theoretical guarantees for \ours in both stochastic linear bandit and general MDP settings. To contextualize the stochastic-linear-bandit part of the result, we compare it with recent work by Lin et al. (2024) \cite{lin2024transformersdecisionmakersprovable}, who analyze SPT in a related bandit setting. Their analysis decomposes the total regret into three components: statistical estimation bias, approximation bias $\epsilon_{\text{real}}$ (which is similar to $\epsilon_{\text{tf}}$ in our setting, see Section \ref{ssec:approx} for details), and intrinsic expert bias $\epsilon_{\text{approx}}$. This comparison is intended to clarify the different behavior of supervised and Bellman-style pretraining under weak random data, rather than to claim a uniform rate dominance across all settings.

\textbf{When offline dataset is collected by random policy.}
A supervised learner trained on random behavior has no improving behavior-action signal: imitation can retain linear cumulative regret. Bellman pretraining instead uses reward information. Our bandit \emph{model-bias component} decreases as $\mathcal O(\sqrt{T\log T/n})$ when fixed problem factors are suppressed, but the total guarantee also contains Proposition~\ref{prop_sb}. The component rate should not be read as replacing the full bound in Proposition~\ref{prop_3}.

\textbf{When offline dataset is collected by expert policy.}
LinUCB trajectories provide stronger action supervision. In that regime, the relevant comparison includes both supervision quality and the finite-data/approximation terms, not a universal rate ordering between methods. The connection to batch fitted-Q analysis \citep{antos2007fitted} concerns the role of offline sample size, not a proof that all horizon or coverage factors are identical.

\vspace{-1ex}
\subsection{A Dataset Criterion for When QTPT Helps}
\vspace{-1ex}
In this section, we analyze the properties of offline datasets and give a sufficient condition under which QTPT can improve over SPT, focusing on how dataset characteristics influence policy learning. Inspired by \citep{kumar2022preferofflinereinforcementlearning}, 
we first define the non-informative context set as follows: 
\begin{definition}\label{def_informative}(\textbf{Non-informative Context})
For an environment $M$ and its optimal Q-function $Q^*_M$, the non-informative context set $G$ contains all $a\in\mathcal{A}$ and $x\in\mathcal{X}$ such that $\nu_M(x,a)\le O(T^{-1})$, where $\nu_M(x,a)=\underset{a'}{\max} Q^*_M(x,a')-Q^*_M(x,a)$.
\end{definition}

Here $G$ is a set of \emph{near-indifferent context--action pairs}, not a set of histories carrying no information about the task. It necessarily includes an optimal action at each context because its action gap is zero. Thus it is not empty; the criterion concerns the gap and visitation of suboptimal actions outside $G$.
For any $(x,a)$ pair in $G$, there is no alternative action $a' \in\mathcal{A}$ such that the value of the optimal $Q$ function is significantly improved. In this situation, simply following the behavior policy does not substantially damage the output policy, and SPT can yield behavior similar to \ours. However, when there exists a potential action for which $\nu_{M}$ is large enough, the Bellman target in Algorithm~\ref{algo:Q} can move the learned policy toward this better action, while SPT continues to imitate the behavior policy. We formally state this sufficient-condition perspective in Proposition~\ref{data_criterion}.

\begin{proposition}\label{data_criterion}
Let $G$ be the non-informative context set of an environment $M$ with optimal $Q$ function $Q^*_{M}$. Assuming there exists a positive constant $c$ such that for $(x,a) \in(\mathcal{X}\times\mathcal{A})\setminus G$, $\nu_{M}(x,a)>c$. Then, under the technical coverage and gap conditions stated in the proof, we have
\[
\mathsf{Subopt}_{\Lambda}(\pi^*,\hat{\pi}_{\textrm{QTPT}})\lesssim \mathsf{Subopt}_{\Lambda}(\pi^*,\hat{\pi}_{\textrm{SPT}}).
\]
\end{proposition}
\vspace{-2ex}

%% file: sections/experiment.tex
\vspace{-1ex}
\section{Numerical Experiments}\label{experiments}
\vspace{-1ex}
We evaluate QTPT on theory-aligned controlled benchmarks, contextual MDPs, controlled ablations, and standard D4RL environments. The main-text presentation is organized around three questions. First, does replacing supervised behavior prediction with Q-target pretraining improve robustness when offline data are weak or suboptimal? Second, what roles do context conditioning and Bellman/TD targets play in this improvement? Third, does the resulting method remain effective beyond the controlled environments used for theoretical calibration?

\vspace{-1ex}
\subsection{Stochastic Linear Bandit}
\vspace{-1ex}
We consider a stochastic linear bandit with dimension $d=5$, arms $A=10$, and horizon $T=200$. At each time step, the agent selects $a_t$ and receives reward $r_t=\langle a_t,\theta^* \rangle+\epsilon_t$, where $\epsilon_t \sim N(0,1.5^2)$ and $\theta^*$ is sampled from $[0,1]^d$. The action set is fixed over time, with actions drawn i.i.d.\ from the same range. We pretrain on two offline datasets: one collected by random actions and the other by LinUCB, each containing $100{,}000$ trajectories. Transformer backbones follow the GPT-2-style architecture used throughout the paper, with $8$ layers, $4$ heads, and embedding dimension $256$.

\begin{figure}[t]
    \centering
    \begin{subfigure}[t]{0.49\textwidth}
        \centering
        \includegraphics[width=\textwidth]{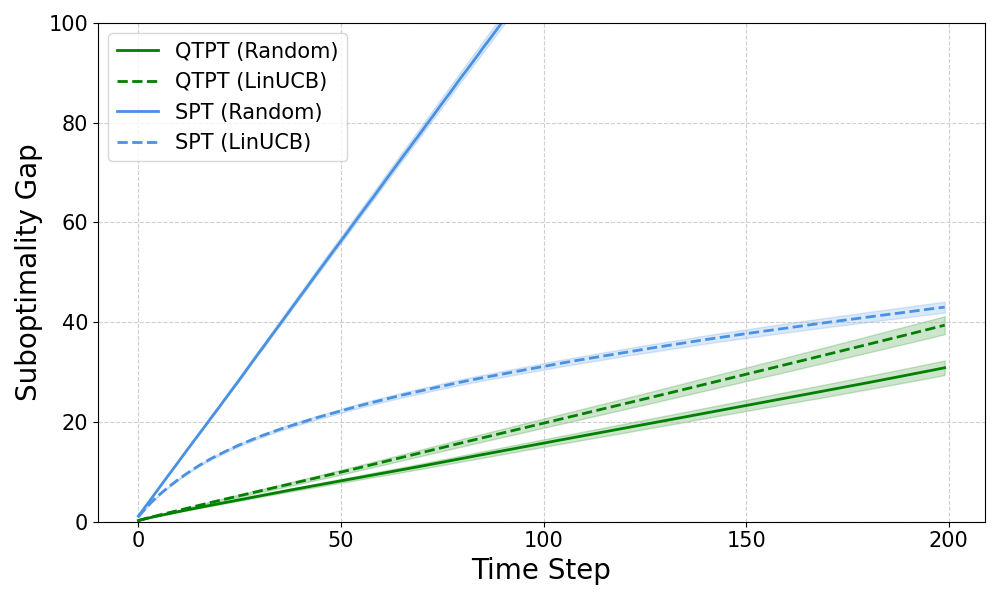}
        \caption{}\label{fig:linear_bandit}
    \end{subfigure}
    \hfill
    \begin{subfigure}[t]{0.49\textwidth}
        \centering
        \includegraphics[width=\textwidth]{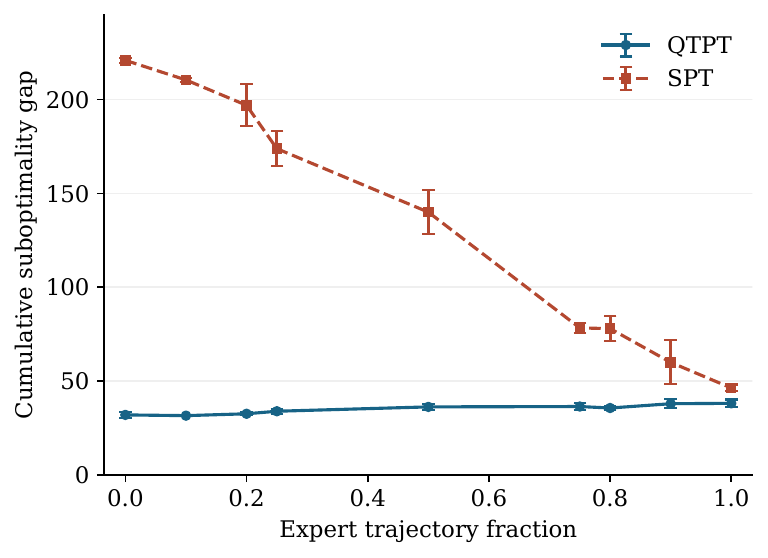}
        \caption{}\label{fig:mix}
    \end{subfigure}
    \caption{\textbf{(a)} Suboptimality gap over time for QTPT and SPT under random and LinUCB pretraining data, averaged over 1000 runs. \textbf{(b)} Supplementary mixture evaluation: mean gap and one standard deviation across five independent training seeds, evaluated on 1,000 task instances. The horizontal axis is the fraction of expert trajectories. These are finite-horizon results at $T=200$, not an empirical verification of asymptotic rates or a basis for extrapolating curve crossings.}
    \label{fig:combined}
\end{figure}

Figure~\ref{fig:linear_bandit} shows that QTPT consistently outperforms SPT on both random and LinUCB datasets. Figure~\ref{fig:mix} further shows that QTPT remains stable as the proportion of expert data decreases, while SPT is highly sensitive to the quality of the offline trajectories. Here weak behavior quality and action coverage are distinct: random actions are poor imitation targets but can supply broad reward/transition evidence. The finite-horizon curves illustrate this distinction rather than measure the theoretical exponents. Appendix~\ref{app:backbone} tests the objective comparison with GPT-2-, Llama-, and Qwen2-style decoders. Additional robustness studies on non-stationary test distributions are reported in Appendix~\ref{ssec:non-stat}, and broader baseline comparisons against CQL, IQL, Q-SFT, and QDT are summarized in Appendix~\ref{app:baseline_compare}.

\vspace{-1ex}
\subsection{Controlled Ablations}
\vspace{-1ex}
We next include controlled ablations to directly answer \emph{why QTPT helps}. To isolate the role of each component, we compare QTPT with two matched variants: \emph{QTPT-MC}, which replaces bootstrapped TD targets with Monte Carlo returns, and \emph{QTPT-NoCtx}, which keeps the TD objective but removes context conditioning. All three models use the same architecture, offline dataset, and evaluation protocol.

\begin{table}[t]
    \centering
    \caption{Controlled ablations on the Darkroom benchmark under matched architectures and offline datasets. Historical submitted evaluation: mean reward with 95\% confidence-interval half-widths over 300 evaluation episodes. These intervals measure episode-level variation, not variability across independent training seeds; higher is better.}
    \label{tab:post_submission_ablation}
    \begin{tabular}{lcc}
        \toprule
        Variant & Random data & Expert data \\
        \midrule
        QTPT (full) & 14.46 $\pm$ 3.27 & 18.29 $\pm$ 4.06 \\
        QTPT-MC & 9.33 $\pm$ 3.01 (-35\%) & 13.77 $\pm$ 3.68 (-25\%) \\
        QTPT-NoCtx & 0.38 $\pm$ 0.61 (-97\%) & 0.34 $\pm$ 0.65 (-98\%) \\
        \bottomrule
    \end{tabular}
\end{table}

Table~\ref{tab:post_submission_ablation} retains the submitted episode-level evaluation, whose uncertainty must not be interpreted as training-seed uncertainty. The point estimates decrease when context is removed and when TD targets are replaced by Monte Carlo returns. They motivate the two mechanisms but do not, on their own, establish seed-level significance. Additional diagnostics below and in Appendix~\ref{app:context_diagnostics} separately examine task-relevant context; the conservative-objective results use their own validation-best protocol and are not pooled with this table.

\vspace{-1ex}
\subsection{Markov Decision Processes: Darkroom and Dark Key-to-door}
\vspace{-1ex}
We evaluate QTPT on Darkroom, Dark Key-to-door, and Miniworld, which are standard contextual decision-making benchmarks for in-context RL \citep{laskin2022context,lee2023supervised}. To keep the main text focused, Figure~\ref{fig:dark_main} highlights Darkroom and Dark Key-to-door, while Miniworld follows the same qualitative trend and is discussed together with the implementation details in Appendix~\ref{sec:exp_details}.

\begin{figure}[t]
    \centering
    \begin{subfigure}[t]{0.49\textwidth}
        \centering
        \includegraphics[width=\textwidth]{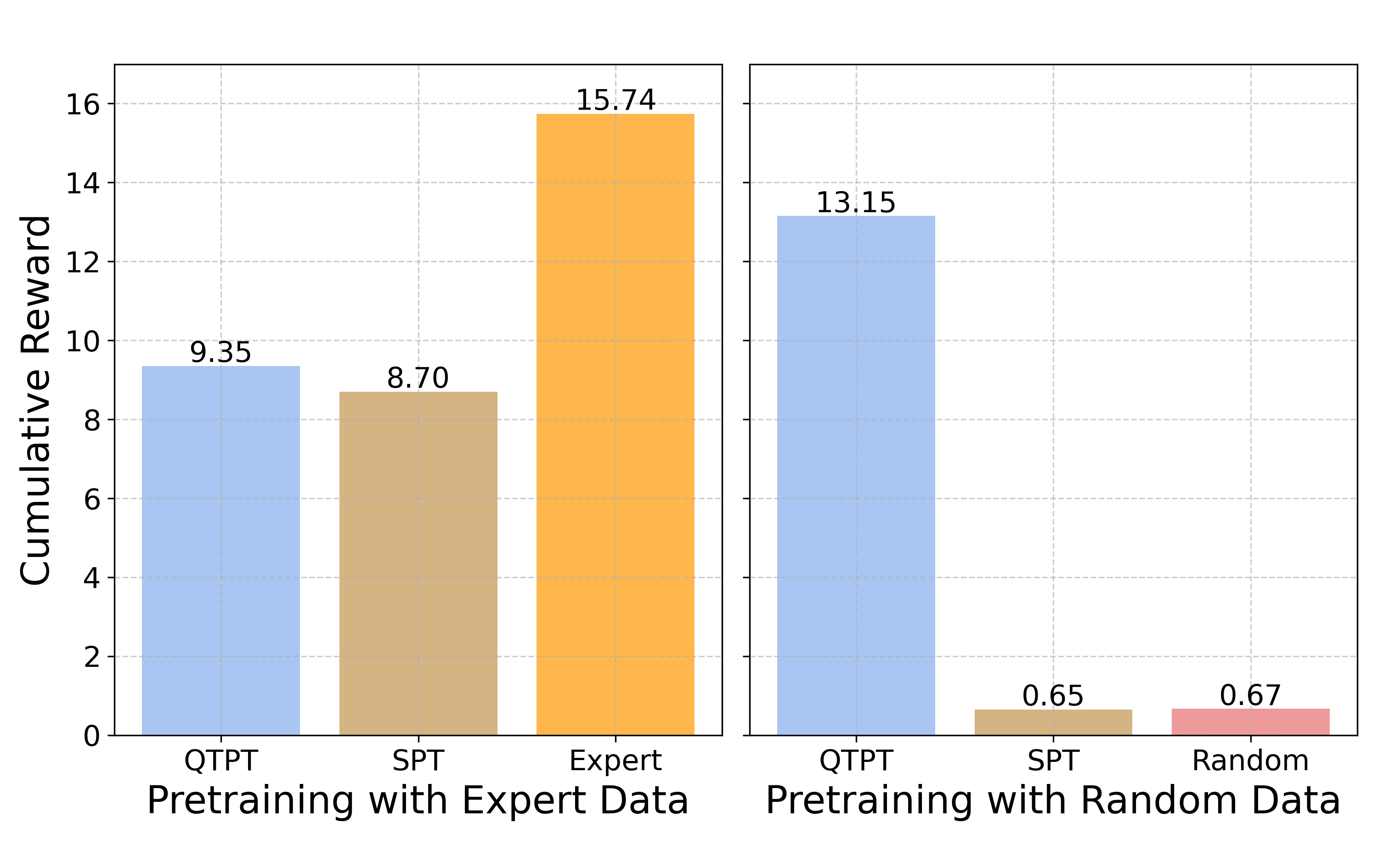}
        \caption{}\label{fig:darkrom}
    \end{subfigure}
    \hfill
    \begin{subfigure}[t]{0.49\textwidth}
        \centering
        \includegraphics[width=\linewidth]{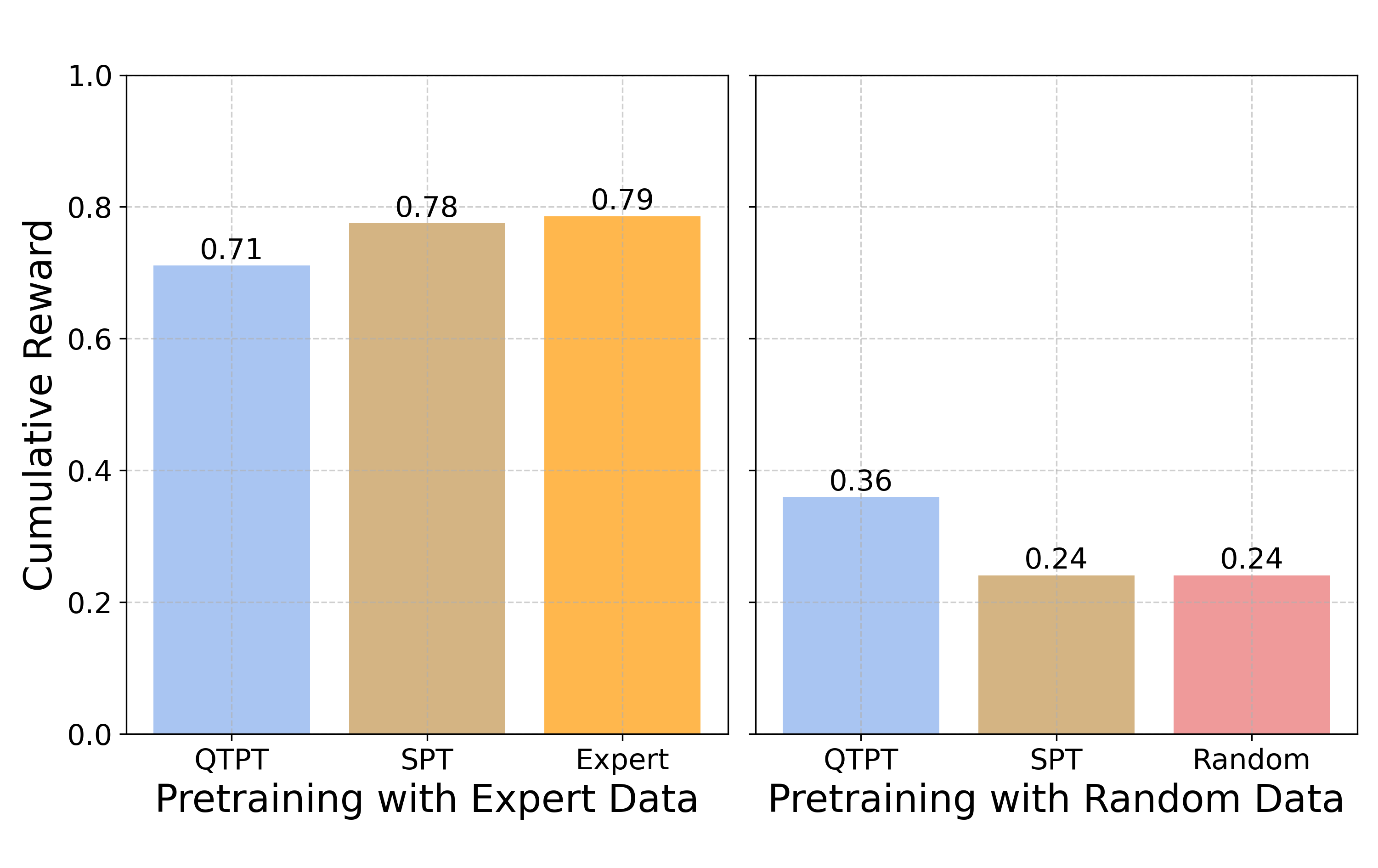}
        \caption{}\label{fig:darkk2d}
    \end{subfigure}
    \caption{\textbf{(a)} Average cumulative reward in Darkroom. \textbf{(b)} Average cumulative reward in Dark Key-to-door. QTPT is usually stronger under weak data, while extremely sparse expert-only settings remain difficult for TD learning.}
    \label{fig:dark_main}
\end{figure}

Across these MDPs, QTPT outperforms SPT in most settings, especially when the offline data are weak or broadly exploratory. An interesting pattern is that random-data pretraining can outperform expert-data pretraining in Darkroom and Miniworld, because random trajectories cover a wider state-action support and thus provide richer signals for Q-learning. The main exception is Dark Key-to-door with expert data, where rewards are extremely sparse and bootstrapped training receives less informative supervision.

\vspace{-1ex}
\subsection{D4RL Extension}
\label{ssec:post_submission}
\vspace{-1ex}
We further extend the study to D4RL Kitchen and AntMaze \citep{fu2020d4rl}. These domains are substantially more complex than the original synthetic environments: they involve higher-dimensional observations, longer horizons, and much sparser reward structure. Table~\ref{tab:d4rl_main} reports the available D4RL extension summaries. Kitchen entries are D4RL normalized scores, while AntMaze entries are active-stage success rates; therefore, comparisons should be made within each task rather than across columns.

\begin{table*}[t]
    \centering
    \small
    \caption{Main-text D4RL extension. QTPT is compared against matched Transformer baselines under the same revised experimental protocol. Kitchen entries report D4RL normalized score (\%), and AntMaze reports success rate (\%). Values report active-stage mean $\pm$ standard deviation across 3--4 seeds; active-stage performance averages non-zero periodic evaluation checkpoints. This conditional checkpoint statistic excludes failed checkpoints and is descriptive, not the conventional final-policy D4RL evaluation. Higher is better; boldface is omitted because no seed-level significance claim is made.}
    \label{tab:d4rl_main}
    \begin{tabular}{lccc}
        \toprule
        Method & Kitchen-partial-v0 & Kitchen-mixed-v0 & AntMaze-umaze-v2 \\
        \midrule
        QTPT (Ours) & 13.23 $\pm$ 4.03 & 14.62 $\pm$ 3.36 & 35.94 $\pm$ 13.52 \\
        SPT & 10.84 $\pm$ 4.18 & 13.25 $\pm$ 2.68 & 26.96 $\pm$ 11.39 \\
        Q-SFT & 11.37 $\pm$ 1.14 & 9.44 $\pm$ 4.86 & 19.05 $\pm$ 9.21 \\
        QDT & 12.58 $\pm$ 4.28 & 7.28 $\pm$ 5.23 & 20.00 $\pm$ 8.66 \\
        \bottomrule
    \end{tabular}
\end{table*}

QTPT has the largest reported mean under this matched Transformer protocol, including 35.94\% active-stage success on \texttt{antmaze-umaze-v2}. The substantial between-seed variation and conditional checkpoint statistic limit the conclusion: these runs demonstrate an implemented extension, not statistically established superiority or a comparison with published D4RL SOTA scores. Appendix~\ref{app:post_submission} uses exactly the same numbers and terminology.

\vspace{-1ex}
\subsection{More Complex Task: Math Reasoning}
\vspace{-1ex}
To test cross-domain breadth, we also adapt QTPT to mathematical reasoning. In this setting we compare QTPO (QTPT with policy optimization) against behavior cloning. Detailed setup and implementation choices are deferred to Appendix~\ref{ssec:math_reasoning}. On GSM8K, QTPO improves over behavior cloning across all mixture ratios by an average of \textbf{1.81\%}, showing that the advantages of value-aware pretraining are not confined to the RL benchmarks studied above.

\subsection{Supplementary Tests of Adaptation and Value Calibration}
Five-seed support-controlled bandit experiments directly expose residual unsupported-action overestimation. QTPT reduces harmful OOD selection relative to naive hard-max QTPT, but adding a CQL penalty produces no statistically reliable incremental improvement in either backup branch (Appendix~\ref{app:ood}). Comparisons with RL$^2$ and PEARL test complementary data-acquisition protocols rather than a universal ranking (Appendix~\ref{app:meta_rl}). The context experiments separate task-coherent history from unrelated extra examples (Appendix~\ref{app:context_diagnostics}). Detailed configurations, uncertainty units, and scope are stated separately for these diagnostics.

%% file: sections/conclusion.tex
\vspace{-1ex}
\section{Limitations}
\vspace{-1ex}
This work focuses on establishing and validating Q-target pretraining as an objective-level alternative to supervised behavior prediction for in-context RL. The theory requires support/coverage conditions on the offline data, so the guarantees should not be interpreted as applying to arbitrary low-coverage datasets. Rather, the intended weak-data regime is one where behavior can be suboptimal or noisy while still providing sufficient state-action coverage for Bellman-style learning.

Our D4RL experiments are designed as fixed-protocol extensions with matched Transformer baselines under the same data and evaluation setup. They are not intended to claim state-of-the-art performance against specialized offline-RL pipelines tuned for D4RL. Finally, our controlled ablations identify context conditioning as necessary for in-context adaptation and Bellman/TD targets as an additional source of robustness; studying alternative backup designs, larger-scale environments, and broader sequence-generation tasks remains important future work.

\vspace{-1ex}
\section{Conclusion and Future Work}
\vspace{-1ex}
We propose Q-target Pretrained Transformers (\ours), which use reward and Bellman-target supervision to learn a context-conditioned value function from offline trajectories. The analysis separates finite-data error from Transformer approximation, with explicit coverage, confidence, and horizon dependence. Controlled experiments show gains over supervised behavior prediction when behavior actions are weak, while supplementary tests examine decoder backbones, meta-RL comparisons, and task-coherent context. The support-controlled diagnostic shows that Double-target and softmax backup mitigate, but do not eliminate, harmful unsupported-action overestimation. Adding a standard CQL penalty does not produce a statistically reliable incremental improvement in that diagnostic. The D4RL extension is descriptive evidence under a matched sequence-model protocol, not a standardized SOTA claim.

While this work establishes a framework for combining Q-learning with ICRL, several directions remain open. First, broader evaluations on larger D4RL suites, newer offline RL benchmarks, and industrial-scale datasets would further test scalability. Second, alternative Bellman backup designs, such as multi-step returns, are promising refinements within the QTPT framework. Finally, it is important to study distribution shifts between training and testing and to verify QTPT on a broader range of NLP and reasoning tasks.

\paragraph{Reproducibility.} The anonymous source code to this project is available in the supplementary. The proofs to the main theory can be
found in \ref{sec:proof}.





%% file: sections/appendix.tex
\onecolumn



\section{Proofs}\label{sec:proof}

\subsection{Proofs Related to The Sample Bias}

\begin{definition}[Rademacher complexity]\label{def:rademacher}
For a generic real-valued function space $\mathcal{F} \subseteq \mathbb{R}^\mathcal{Y}$ and $n$ fixed data points $Y=\{y_1,\cdots,y_n\}\in\mathcal{Y}^n$, the empirical Rademacher complexity is defined as follows: 
\[
\hat{\mathcal{R}}_Y(\mathcal{F})=\mathbb{E}_{\omega}\Big[\underset{f\in \mathcal{F}}{\sup} \frac{1}{n} \sum_{i=1}^n \omega_i f(y_i)|Y\Big],
\]
where $\omega_i$ are independent and identically distributed Rademacher random variables, ranging from $-1$ to $1$ and, thus, representing randomly assigned labels. 
Let $\rho$ be the distribution of $Y$, we further define the Rademacher complexity $\mathcal{R}^{\rho}_{n}(\mathcal{F}):=\mathbb{E}_{\rho}[\hat{\mathcal{R}}_{Y}(\mathcal{F})]$, which measures the complexity of $\mathcal{F}$ by the degree to which the functions in class $\mathcal{F}$ are correlated with the random noise $w_i$. 
\end{definition}

This definition follows a similar structure in \cite{duan2021risk}. 
The Rademacher complexity provides a measure of the complexity of our function class, which will be crucial for bounding the sample bias.

\subsubsection{Proof of Proposition \ref{prop_sb}}

\begin{proof}
On the one hand, for a fixed $Q\in \hat{\mathcal{Q}}$ , considering the union bound, by Lemma G.1 and Lemma G.6 in \cite{duan2021risk}, 
\[
\Pr \Big[|\mathcal{L}_n(Q)-\mathcal{L}(Q)|\geq \epsilon_n\Big]\leq \delta,
\]
where $\epsilon_n=B_{Q}^2\sqrt{\frac{2\ln(2T/\delta)}{n}}+2B_{Q}\mathcal{R}^{\pi_\beta}_{n}(\hat{\mathcal{Q}})+B_{Q}^2$. 
Thus, with probability at least $1-\delta$, $|\mathcal{L}_n(Q)-\mathcal{L}(Q)|\leq \epsilon_n$, and 
\[
\mathcal{L}(\hat{Q}^*)\leq \mathcal{L}_n(\hat{Q}^*)+\epsilon_n\leq \mathcal{L}_n(Q^*)+\epsilon_n\leq \mathcal{L}(Q^*)+2\epsilon_n,
\]
since $\mathcal{L}(Q^*)=0$, $\mathcal{L}(\hat{Q}^*)\leq 2\epsilon_n$. \\
On the other hand,  by Assumption \ref{assump:bound}, it implies that $d^{\pi_\beta}_{M,t}(x_t,a_t)\geq L^{-1}$, when $d^{\pi_\beta}_{M,t}(x_t,a_t)>0$.\\
For any $(t,x_t,a_t)$, 
\[
\frac{d^{\pi^*}_{M,t}(x_t,a_t)}{d^{\pi_{\beta}}_{M,t}(x_t,a_t)}\leq \frac{1}{d^{\pi_{\beta}}_{M,t}(x_t,a_t)}\leq L.
\]
Thus, using Performance Difference Lemma by Lemma 6.1 in \cite{kakade2002approximately}, we have

\begin{equation*}
    \begin{array}{rl}
        \mathsf{Subopt}_{\Lambda}(\pi^*,\pi_{\hat{Q}^*})&=  \mathbb{E}_{M\sim \Lambda}\sum^T_{t=1}\mathbb{E}_{d^{\pi^*}}\Big[ |(\Gamma \hat{Q}^*)(x_t,a)-\hat{Q}^*(x_t,a)| \Big] \\
         & \leq \sqrt{T\sum^T_{t=1}\mathbb{E}_{d^{\pi^*}}\Big[|\hat{Q}^*(x_t,a)-(\Gamma \hat{Q}^*)(x_t,a)|^2\Big]}\\
         & \leq T\sqrt{L}\sqrt{\mathcal{L}_n(\hat{Q}^*)}\\
         & \leq \sqrt{L}T\sqrt{2\epsilon_n}\\
         & \leq \sqrt{L}\Big[ 2T^2(\frac{2\log (2T/\delta)}{n})^{1/4}+\sqrt{2}T^{3/2}\sqrt{\mathcal{R}^{\pi_\beta}_{n}(\hat{\mathcal{Q}})}+TB \Big],
    \end{array}
\end{equation*}
using Proposition 6.1 in literature \cite{duan2021risk}, 
we get the final results as  
\[
\mathsf{Subopt}_{\Lambda}(\pi^*,\pi_{\hat{Q}^*})\leq \mathcal{O}\Big(\sqrt{L} T^2\Big[(\frac{2\log (2T/\delta)}{n})^{1/4}+\sqrt{\max\{\sqrt{\frac{\log |\mathcal{Q}|}{n}} ,\frac{\log |\mathcal{Q}|}{n}\}}\Big]+\sqrt{L}TB\Big).
\]

\end{proof}

\subsection{Proofs Related to the Model Bias}

\subsubsection{Proof of Proposition \ref{linearbandit} }
\begin{proof}
At iteration $k$, with current parameters $w^{(k)}$, compute $y_t^{i,(k)}=r_t+\max_{a'\in \mathcal{A}_{t+1}}Q_{w^{(k)}}(D_t^i,s^i_{t+1},a'),$
note that $y^{i,(k)}_t$ is now a fixed constant w.r.t. $w$. 
Thus,  {\ours} at round $t-1$ solves the ordinary least-square problem
$\hat{w}_{t-1}=\arg \min_w \sum^n_{i=1}\sum_{u=1}^{t-1} (\langle a^i_u,w\rangle-r^i_u)^2$.
A greedy policy then selects $a_t=\arg\max_{a\in\mathcal{A}_t}\langle a,\hat{w}_{t-1}\rangle$.
Define the estimation error term ${EE}_t=\sup_{a\in\mathcal{A}_t}|\langle a,w^*-\hat{w}\rangle|$. 
We get the solution $\hat{w}_{t-1}=G^{-1}_{t-1}\sum_{i=1}^n\sum_{u=1}^{t-1}a^i_u r^i_u$.\\
On the one hand, by Theorem 2 in Abbasi-Yadkori et al. (2011)\cite{NIPS2011_e1d5be1c}, for zero-mean, $\sigma$-sub-aussian noise $\epsilon_u$, with probability at least $1-\delta$, we have
\[
|\hat{w}_{t-1}-w^*\|_{G_{t-1}}\leq \sigma \sqrt{2\log \frac{\det(G_{t-1})^{1/2}}{\det(\alpha nI)^{1/2}\delta}}=\beta_{t-1}\]
On the other hand, for any $a$, 
\[|Q_{\hat{w}}-Q^*|=|\langle a,\hat{w}-w^*\rangle|\leq \|a\|_{G_{t-1}^{-1}}\|\hat{w}-w^*\|_{G_{t-1}}\leq \frac{\beta_{t-1}}{\sqrt{\lambda_{min}(G_{t-1})}}\leq \frac{\beta_{t-1}}{\sqrt{\alpha n(t-1)}},\]
since $\det(G_{t-1})=\mathcal{O}((\alpha n)^d(t-1)^d)$, we have $\beta_{t-1}=\mathcal{O}(\sigma\sqrt{d\ln t+\log(1/\delta)})$. \\
Thus, \[EE_t=\sup|Q_{\hat{w}}-Q^*|\leq \mathcal{O}(\sigma\sqrt{\frac{d\ln t+\log(1/\delta)}{\alpha nt}}).\]


The cumulative suboptimality over $T$ steps as 
\[\mathbb{E}_{M\sim \Lambda}\Big[\sum_{t=1}^T (\langle a^*,w^*\rangle-\langle a_t,\hat{w}\rangle)\Big],\]
then we have 
\[\langle a^*,w^*\rangle-\langle a_t,\hat{w}\rangle=\Big[ 
\langle a^*,w^*\rangle-\langle a^*,\hat{w}\rangle \Big]
+\Big[ \langle a^*,\hat{w}\rangle-\langle a_t,\hat{w}\rangle \Big]
\leq MB_t,\]
Finally, we get that \[ 
\mathsf{Subopt}(\pi_{\hat{Q}^*},\pi_{\tilde{Q}})
\leq 2\mathbb{E}[\sum_{t=1}^T EE_t]+\epsilon_{\text{tf}}\leq \epsilon_{\text{tf}}+\mathcal{O}(\sigma\sqrt{T}\sqrt{\frac{d\log (nT)+\log(1/\delta)}{\alpha n}}),\]
where $\epsilon_{\text{tf}}$ is the Transformer approximation error (see Section \ref{ssec:approx}). 
\end{proof}

\subsubsection{Proof of Proposition \ref{prop_mb}}
\begin{proof}
For simplicity, we abbreviate $\hat{Q}^{\pi}_{M,t}$ as $\hat{Q}_t$. 
Let $\hat e^{\mathrm{TD}}_t(x, a)
= \hat{Q}_t(x, a) - r - 
\mathbb{E}_{x_{t+1}}\left[\max_{a'} \hat{Q}_{t+1}(x_{t+1}, a')\right]$.
Similarly, using the definitions of $\tilde e^{\mathrm{TD}}_t$ and $e^{\mathrm{TD},*}_t$, we can express $\hat{Q}$ to $\tilde{Q}$ and $Q^*$ in terms of $\hat e^{\mathrm{TD}}_t$.
For function $\hat{Q}=(\hat{Q},\cdots,\hat{Q}_T)$, define
\[\epsilon(\hat{Q})=\frac{1}{T}\sum_{t=1}^{T}\mathbb{E}\left[\hat{Q}_t - r -\mathbb{E}_{x_{t+1}}\max_{a'} \hat{Q}_{t+1}(x_{t+1}, a')\right]^2 = \frac{1}{T}\sum_{t=1}^{T} \mathbb{E} \hat e^{\mathrm{TD}}_t(x, a)^2.\]
Similarly, using the definition of $\epsilon(\tilde{Q})$, we can express $\hat{Q}$ to $\tilde{Q}$ in terms of $\epsilon(\hat{Q})$.  
Since $\pi_{\tilde{Q}}(x) = \arg \underset{a'}{\max} \ \tilde{Q}(x, a')$, we can show that
\[V_M^{\pi_{\hat{Q}^*}}(x)-V_M^{\pi_{\tilde{Q}}}(x) \leq V^{\pi_{\hat{Q}^*}}(x)-\tilde{Q}_1(x,\pi_{\hat{Q}^*}(x))+\tilde{Q}_1(x,\pi_{\tilde{Q}}(x))-V_M^{\pi_{\tilde{Q}}}(x).\]
For any policy $\pi$, we have
\begin{equation*}
    \begin{split}
        & \tilde{Q}_1(x_1, \pi(x_1)) - V_{M,1}^{\pi}(x_1) \\
= & \mathbb{E}\left[\sum_{t=1}^{T} \Big(\tilde{Q}_t(x_t, a_t)- \mathbb{E}_{\pi}[\tilde{Q}_{t+1}(x_{t+1}, a_{t+1}) + r_t|x_t,a_t] \Big) \Big| x_1, \pi  \right] \\
= & \mathbb{E}\left[ \sum_{t=1}^T \left( \tilde{Q}_t(x_t, a_t) - r_t - \tilde{Q}_{t+1}(x_{t+1}, a_{t+1}) \right) \Big| x_1, \pi \right].
    \end{split}
\end{equation*}
We can show that 
\[\Big|\mathbb{E}\Big[\sum^T_{t=1}\tilde e^{\mathrm{TD}}_t(x_t,a_t)|x_1,\pi\Big]\Big|\leq \sqrt{T\sum^T_{t=1} \mathbb{E}\Big[\tilde e^{\mathrm{TD}}_t(x_t,a_t)^2|x_1,\pi\Big]}\leq T \sqrt{\epsilon(\tilde{Q})},\]
which implies that
\[
\tilde{Q}_1(x_1,\pi_{\tilde{Q}}(x_1))-V_{M,1}^{\pi_{\tilde{Q}}}(x_1)=\mathbb{E}\Big[\sum^T_{t=1}
\tilde e^{\mathrm{TD}}_t(x_t,a_t)|x_1,\pi_{\tilde{Q}}\Big],
\]
\[
\tilde{Q}_1(x_1,\pi_{\hat{Q}^*}(x_1))-V_{M,1}^{\pi_{\hat{Q}^*}}(x_1)\geq \mathbb{E}\Big[\sum^T_{t=1}
\tilde e^{\mathrm{TD}}_t(x_t,a_t)|x_1,\pi_{\hat{Q}^*}\Big].
\]
Hence,  
\[V_{M,1}^{\pi_{\hat{Q}^*}}(x)-V_{M,1}^{\pi_{\tilde{Q}}}(x)\leq \mathbb{E}\Big[\sum^T_{t=1}\tilde e^{\mathrm{TD}}_t(x_t,a_t)|x_1,\pi_{\tilde{Q}}\Big]- \mathbb{E}\Big[\sum^T_{t=1}\tilde e^{\mathrm{TD}}_t(x_t,a_t)|x_1,\pi_{\hat{Q}^*}\Big]\leq 2T \sqrt{\epsilon(\tilde{Q})}.\]
since $\hat e^{\mathrm{TD}}_t(x,a)^2-\tilde e^{\mathrm{TD}}_t(x,a)^2\in 
[-2 T^2,2 T^2]
$, using Massart Finite Class Lemma,  with probability at least $1-\delta$, there exist a constant $c$ such that for any $\pi$, 
\begin{equation*}
    \begin{array}{rl}
        & \mathbb{E}\tilde e^{\mathrm{TD}}_{t}(x,a)^2-\mathbb{E}\hat e^{\mathrm{TD}}_{t}(x,a)^2 \\
        \leq & 
        \frac{1}{n}\underset{(x_t,a_t,r_t,x_{t+1})\sim D_t}{\sum}(\tilde e^{\mathrm{TD}}_{t}(s,a)^2-\hat e^{\mathrm{TD}}_{t}(x,a)^2)
        \\ &\quad + 2 R_n({(\tilde e^{\mathrm{TD}}_{t})^2-(\hat e^{\mathrm{TD}}_t)^2|\hat{Q}_t})+2T^2\sqrt{\frac{2\log(2/\delta)}{n}}\\
         \leq & 
         2T^2\sqrt{\frac{2\log(2/\delta)}{n}}+cT R_n(\hat{Q}^{~}_t) \\
         \leq & 
         2T^2\sqrt{\frac{2\log(2/\delta)}{n}}+2cT^2 \frac{\sqrt{2\log|\mathcal{Q}_t|}}{n}.
    \end{array}
\end{equation*}
Since 
\[
\epsilon(\tilde{Q})-\epsilon(\hat{Q})=\frac{1}{T}\sum^T_{t=1} \Big(\mathbb{E}\tilde e^{\mathrm{TD}}_{t}(x,a)^2-\mathbb{E}\hat e^{\mathrm{TD}}_{t}(x,a)^2\Big), 
\]
since $\epsilon(\hat{Q})$ indicates the capacity of Transformers approximation (see Section \ref{ssec:approx} for details), we can get that 
\begin{equation*}
    \begin{array}{rl}
        \epsilon(\tilde{Q}) \leq &  \epsilon(\hat{Q})
        +2T^2\sqrt{\frac{2\log(2T/\delta)}{n}}+2cT\underset{t=1}{\overset{T}{\sum}}\frac{\sqrt{2\log|\hat{\mathcal{Q}}_t|}}{n}\\
         \leq & \epsilon_{\text{tf}} 
        +  2T^2\sqrt{\frac{2\log(2T/\delta)}{n}}
        +2cT\underset{t=1}{\overset{T}{\sum}}\frac{\sqrt{2\log|\hat{\mathcal{Q}}_t|}}{n},
    \end{array}
\end{equation*}
Thus,  
\begin{equation*}
    \begin{array}{rl}
        & V_{M,1}^{\pi_{\hat{Q}^*}}(x_1)-V_{M,1}^{\pi_{\tilde{Q}}}(x_1)\\
        \leq & 2T\sqrt{\epsilon_{\mathrm{tf}}}+T^2\Big( \sqrt{2}(\frac{2\log (2T/\delta)}{n})^{1/4}+\sqrt{2c\frac{\sqrt{2\sup \log|\hat{\mathcal{Q}}_t|}}{n}} \Big) \\
         = & 2T\sqrt{\epsilon_{\mathrm{tf}}}+\mathcal{O}\Big( T^2 \Big(n^{-1/2}+\left(2\log(2T/\delta)/n\right)^{1/4}\Big)\Big).
    \end{array}
\end{equation*}
\end{proof}

\subsubsection{Proof of Proposition \ref{prop_3}}

\begin{proof}
Since 
\[
\begin{aligned}
&\int \Lambda_{\mathrm{test}}(M)\,
\mathbb E_{x_1\sim\eta_M}\!\left[V^*_{M,1}(x_1)-V^{\pi_{\tilde Q}}_{M,1}(x_1)\right]\,\mathrm dM\\
&\qquad\leq \mathcal C\int \Lambda_{\mathrm{train}}(M)\,
\mathbb E_{x_1\sim\eta_M}\!\left[V^*_{M,1}(x_1)-V^{\pi_{\tilde Q}}_{M,1}(x_1)\right]\,\mathrm dM,
\end{aligned}
\]
thus 
\[\mathsf{Subopt}_{\Lambda_{\text{test}}}(\pi^*,\pi_{\tilde{Q}})\leq \mathcal{C} \cdot \mathsf{Subopt}_{\Lambda_{\text{train}}}(\pi^*,\pi_{\tilde{Q}}).\]
For the bandit case apply Proposition~\ref{prop_sb} and Proposition~\ref{linearbandit} with confidence parameter $\delta/2$ each; for the MDP case use Proposition~\ref{prop_sb} and Proposition~\ref{prop_mb} instead. In either case the union bound gives probability at least $1-\delta$ for their simultaneous validity. Add the two bounds using Eq.~\eqref{eq_rev_dif}, retaining the approximation terms, and then apply the environment-density ratio $\mathcal C$. 
\end{proof}

\subsection{Proofs Related to the Comparison of \ours and SPT}

\subsubsection{Proof of Proposition \ref{data_criterion}}

At first, we need to prove a result on non-informative context states, which is inspired by Lemma B.10 in \citep{kumar2022preferofflinereinforcementlearning}. 
For simplify, when $G$ be the non-informative context set, define $G_{\mathcal{A}},G_{\mathcal{X}}$ be the separate sets to represent all unique context states and actions present in $G$, i.e. $G_{\mathcal{A}}=\{a|\exists x, (x,a)\in G\}$, $G_{\mathcal{X}}=\{x|\exists a, (x,a)\in G\}$. 

For an environment $M$ and $\forall a\in \mathcal{A}\setminus G_{\mathcal{A}}$, $\nu_M(x,a)\simeq \Delta_M(x)$.
For $(x,a)\in G$, there exists a constant $\epsilon>0$ that satisfies $\nu_M(x,a)\leq \frac{\epsilon}{T}$. 

\begin{lemma}
\label{lemma_critical}
Consider a fixed environment $M$, 
policy $\hat{\pi}_{\text{QTPT}}$ is obtained from the \ours algorithm. 
There exists a non-informative context set $G$ satisfying $|\mathcal{X}\setminus G_{\mathcal{X}}|\geq n_0$,  
and let $\forall a\in \mathcal{A}$,  $\frac{d^{\pi^*}_{M}(a|x)}{d^{\pi_\beta}_M(a|x)}\leq \alpha_0<2$, 
where $d^{\pi_{\beta}}_M(a|x)$ is the conditional action distribution at context state $x$, $\alpha_0$ is a constant. 
$\tilde{Q}$ is the learned Q-value, since $\tilde{Q}$ can only differ from $Q^*$, there exists $\epsilon_0>0$, s.t. $|\tilde{Q}(x,a)-Q^*(x,a)|\leq \epsilon_0$. 
Then, the probability that the policy $\hat{\pi}_{\text{QTPT}}$ doesn't choose the action in $G_{\mathcal{A}}$ at context state $x$ is upper bounded as 
\begin{align*}
    &\mathbb{P}[\hat{\pi}_{\text{QTPT}}(x)\notin G_{\mathcal{A}}]\\
    \leq& \exp \Big(-n_0 \frac{\alpha_0^2}{2(\alpha_0-1)}\times [\frac{|G_{\mathcal{A}}|n_{\Delta}(x)}{n_0}-\frac{1}{\alpha_0}]^2\Big)
\end{align*}
where $n(x,a)$ be the expected number of visit $(x,a)$ in dataset $D$, $n_{\Delta}$ corresponds to the maximum value of $n(x,a)$ s.t. \[\frac{\epsilon}{T}\geq \Delta_M(x)-2\epsilon_0.\]
\end{lemma}
\begin{proof}
\begin{equation*}
    \begin{array}{rl}
       \mathbb{P}[\hat{\pi}_{\text{QTPT}}(x)\notin G_{\mathcal{A}}]  =&\mathbb{P}_M[\exists a\notin G_{\mathcal{A}}, \, \forall a_g\in G_{\mathcal{A}}, \, \hat{Q}(x,a)\geq \hat{Q}(x,a_g)] \\
          \leq  &\mathbb{P}_M[\exists a\notin G_{\mathcal{A}}, \forall a_g\in G_{\mathcal{A}}, \, \hat{Q}(x,a)-Q^{*}(x,a)\geq \hat{Q}(x,a_g)-Q^*(x,a_g)\\&+\Delta_M(x)-\epsilon/T] \\
          \leq &\mathbb{P}_M[\cap_{a_g\in G_{\mathcal{A}}} \{\epsilon/T\geq \Delta_M(x)-2\epsilon_0\} ] \\
          \leq &\mathbb{P}_M[\cap_{a\in G_{\mathcal{A}}} \{n(x,a)\leq n_{\Delta}\} ]
    \end{array}
\end{equation*}
where $n_{\Delta}$ corresponds to the maximum value of $n(x,a)$ s.t. 
\[\frac{\epsilon}{T}\geq \Delta_M(x)-2\epsilon_0\]
consider the action sampled from $d_M^{\pi_{\beta}}(a|x)$ belong to the set $G_{\mathcal{A}}$ or not, 
\begin{equation*} \begin{array}{rl} \mathbb{P}_M[\cap_{a_g\in G_{\mathcal{A}}} \{n(x,a)\leq n_{\Delta}\} ] & \leq \mathbb{P}_M[\sum_{a_g\in G_{\mathcal{A}}} n(x,a)\leq |G_{\mathcal{A}}|n_{\Delta}(x)]\\
&\leq \mathbb{P}_M[\frac{\sum_{a_g\in G_{\mathcal{A}}} n(x,a)}{n_0}\leq \frac{|G_{\mathcal{A}}|n_{\Delta}(x)}{n_0}]\\\ &\leq \exp \Big(-n_0 KL\Big(Bern(\frac{|G_{\mathcal{A}}|n_{\Delta}(x)}{n_0})||Bern(1/\alpha_0)\Big)\Big)\end{array}\end{equation*}
since $KL(p+\epsilon||p)\geq \frac{\epsilon^2}{2p(1-p)}$, if $p\geq 1/2$,  $p=1/\alpha_0$,  we have 
\[ \mathbb{P}_M[\cap_{a\in G_{\mathcal{A}}} \{n(x,a)\leq n_{\Delta}(x)\} ]\leq \exp \Big(-n_0 \frac{\alpha_0^2}{2(\alpha_0-1)}\times [\frac{|G_{\mathcal{A}}|n_{\Delta}(x)}{n_0}-\frac{1}{\alpha_0}]^2\Big)\]
\end{proof}
Combining the result of Lemma \ref{lemma_critical}, for an appropriate value of $c$ and under the stated coverage and gap conditions, we obtain Proposition \ref{data_criterion}, which compares \ours and SPT when a few informative context states appear in the given trajectories. 

\begin{proof}

Suppose there exists $p_M\in [0,1]$ s.t. for any policy $\pi$, the average occupancy of states which is not in non-informative set $G$ is bounded as 
\[
\sum_{t=1}^T\sum_{x\notin G_{\mathcal{X}}} d^{\pi}_{M,t}(x)\leq p_M.
\]
Any given environment $M\sim \Lambda $ such that
$\sup_{x\in \mathcal{X}}\{d^{\pi^*}_M(x)/d^{\pi_\beta}_M(x)\}=m$ be a generalization parameter which is independent of $T$, 
and $L_0\leq 1+1/n$. 
A non-informative context set $G$ is given such that $|\mathcal{X}\setminus G_{\mathcal{X}}|\geq n_0$. 
For learning algorithm that return a policy $\hat{\pi}^*$, the suboptimality is given: 
\begin{equation*}
    \begin{array}{rl}
        & \mathsf{Subopt}(\pi^*,\hat{\pi}^*) \\
        =&\sum_{x} d^{\pi^*}_M(x)
        (V^{\pi^*}_{M,1}(x)-V^{\hat{\pi}^*}_{M,1}(x)) \\
        \leq & m\Big[\underbrace{\sum_{x\notin G_{\mathcal{X}}}d_M^{\hat{\pi}^*}(x) (Q^*(x;\pi^*)-Q^*(x;\hat{\pi}^*))}_{term\ 1}+\underbrace{\sum_{x\in G_{\mathcal{X}}}d_M^{\hat{\pi}^*}(x) (Q^*(x;\pi^*)-Q^*(x;\hat{\pi}^*))}_{term\ 2} \Big], 
    \end{array}
\end{equation*}
On the one hand, consider the part of $x\notin G_{\mathcal{X}}$ (named term 1), we can construct a new MDP where $r(x,a)=r(x,\pi^*(x))$ for all actions and $\forall x\in G_{\mathcal{X}}$. Denote the suboptimality of this MDP be $\mathsf{Subopt}_{x\notin G_{\mathcal{X}}}(\pi^*,\hat{\pi}^*)$, then we have 
\[
\mathsf{Subopt}(\pi^*,\hat{\pi}^*)=\mathsf{Subopt}_{x\notin G_{\mathcal{X}}}(\pi^*,\hat{\pi}^*)
\]
This can be bounded as in Proposition \ref{prop_mb} , except with the dependence on all context states $\mathcal{X}$ replaced by $\mathcal{X} \setminus G_{\mathcal{X}}$ . We get that 
\[
term\ 1\lesssim p_MT^2(\frac{1}{\sqrt{n}}+\frac{\sqrt[4]{\log T}}{\sqrt[4]{n}})
\]
On the other hand, consider the part of $x\in G_{\mathcal{X}}$ (named term\, 2), 
\[
term\, 2=\sum_{x\in G_{\mathcal{X}}}\sum_{a\in G_{\mathcal{A}}}W(x,a,\pi^*)+\sum_{x\in G_{\mathcal{X}}}\sum_{a\notin G_{\mathcal{A}}}W(x,a,\pi^*)
\]
where $W(x,a,\pi^*)=d_M^{\hat{\pi}^*}(x,a) (Q^*(x;\pi^*)-Q^*(x,a))$. 
Since 
\begin{align*}
& \sum_{a\in G_{\mathcal{A}}}W(x,a,\pi^*)\leq d_M^{\hat{\pi}^*}(x) \mathbb{P}[\hat{\pi}^*(x)\in G_{\mathcal{A}}]\cdot \frac{\epsilon}{T},\\
& \sum_{a\notin G_{\mathcal{A}}}W(x,a,\pi^*)\leq d_M^{\hat{\pi}^*}(x)\mathbb{P}[\hat{\pi}^*(x)\notin G_{\mathcal{A}}]\cdot c,
\end{align*}
using Lemma \ref{lemma_critical}, and note that in this case, we are interested in the setting where $L_0\simeq 1+\mathcal{O}(1/n)$, 
we can get that 
\[
\sum_{x\in G_{\mathcal{X}}} d_M^{\hat{\pi}^*}(x)\mathbb{P}[\hat{\pi}^*(x)\notin G_{\mathcal{A}}]\cdot c\lesssim (1-p_M) f_{\text{QTPT}}(n,c),
\]
where $f=c\cdot\exp(-n_0\times \frac{(n+1)^2}{2n}\times [\frac{|G_{\mathcal{A}}|n_{\Delta}(x)}{n_0}-\frac{n+1}{n}]^2)$ is an exponential function of $-n$. Meanwhile, the corresponding term for SPT is 
\begin{align*}
&\sum_{s\in G_{\mathcal{X}}} d_M^{\hat{\pi}^*}(s)\mathbb{P}[\hat{\pi}^*(s)\notin G_{\mathcal{A}}]\cdot c\\
\lesssim & (1-p_M)
c\cdot \frac{1}{n}\\
=&(1-p_M) g_{\text{SPT}}(n,c).
\end{align*}
By controlling $c$ and $p_M$ for some appropriate values, we have the comparison of \ours and SPT as follows: 
\[
\frac{\mathsf{Subopt}_{\Lambda}(\pi^*,\hat{\pi}_{\text{QTPT}})}{\mathsf{Subopt}_{\Lambda}(\pi^*,\hat{\pi}_{\text{SPT}})}\simeq \frac{m \Big[ p_MT^2(\frac{1}{\sqrt{n}}+\frac{\sqrt[4]{\log T}}{\sqrt[4]{n}})+(1-p_M)(f+\epsilon/T) \Big] }{p_M\frac{T}{n}+(1-p_M)g},
\]
since $T$ is finite, combining $f\rightarrow \mathcal{O}(e^{-n}),g\rightarrow \mathcal{O}(1/n)$,
setting $p_M=\frac{1}{T\sqrt{T}}$, we get that 
\begin{align*}
    \frac{\mathsf{Subopt}_{\Lambda}(\hat{\pi}_{\text{QTPT}})}{ \mathsf{Subopt}_{\Lambda}(\hat{\pi}_{\text{SPT}})} & \simeq \frac{m\Big[ p_M T^2(\mathcal{O}(n^{1/2})+\mathcal{O}(n^{3/4})) +(1-p_M)fn \Big]}{p_M T+(1-p_M)gn} \\
    & \simeq \frac{m\Big[ \sqrt{T}(\mathcal{O}(n^{1/2})+\mathcal{O}(n^{3/4})) +(1-p_M)fn \Big]}{1/\sqrt{T}+(1-p_M)gn} \\
    & \simeq \frac{(1-p_M)fn}{(1-p_M)gn} \\
    & \simeq \mathcal{O}(e^{-n}/n)
\end{align*}
so we get that
\[
\mathsf{Subopt}_{\Lambda}(\hat{\pi}_{\text{QTPT}})\lesssim \mathsf{Subopt}_{\Lambda}(\hat{\pi}_{\text{SPT}}).
\]
\end{proof}

\subsection{Transformer Realization of the QTPT Update}\label{ssec:approx}

\begin{proposition}\label{tran_approx}
(\textbf{Transformer Realization of One QTPT Bellman Update})
For bounded Q-value predictions and a tokenized transition history, there exists a Transformer block of appropriate scale that constructs the TD target and implements the relaxed QTPT update
\[
Q^{+}(x_t,a_t)=(1-\alpha)Q(x_t,a_t)+\alpha\left(r_t+\max_{a'}Q(x_{t+1},a')\right),\quad \forall t\in[T].
\]
This proposition establishes architectural realizability of the update operation. It does not by itself imply that a fixed Transformer class has zero approximation error or that repeated optimization converges to the ideal Bellman minimizer; those effects are represented separately by $\epsilon_{\mathrm{tf}}$ in the model-bias bound.
\end{proposition}

\begin{proof}

\textbf{Embedding and Extraction Mappings}

For each $t\in [T]$, we construct two tokens 
\[
h_{2(t-1)}=\begin{bmatrix} s_t\\ \hdashline 0_{|\mathcal{A}|+1}\\ \hdashline pos_{2(t-1)}\end{bmatrix}
=\begin{bmatrix}h^a_{2(t-1)}\\h^b_{2(t-1)}\\h^c_{2(t-1)} \end{bmatrix},\quad 
h_{2t-1}=\begin{bmatrix}0_{|\mathcal{S}|}\\ \hdashline a_t\\r_t\\ \hdashline pos_{2t-1}\end{bmatrix}
=\begin{bmatrix}h^a_{2t-1}\\h^b_{2t-1}\\h^c_{2t-1} \end{bmatrix}
\]
where $pos$ is the position embedding , $s_t,a_t$ are represented using one-hot embedding, $h^b_{2t-1}$ is used to store the policy at time $t$ given current state $s_t$. 
we add an empty token $h_{2T}=[0_{|\mathcal{S}|+|\mathcal{A}|+1}\quad pos_{2T}]^T$ to store intermediate calculations. We also include in the tokens the position embedding $pos_i=(i,i^2,1)^T$ for $i\in [2T]$. We define the token matrix $H_t=[h_0,\cdots,h_{2t-1}]\in \mathbb{R}^{D_{dim}\times 2t}, D_{dim}=|\mathcal{S}|+|\mathcal{A}|+4$ for all $t\in [T]$. 

\textbf{Pretraining} 

During pretraining the Transformer $TF_{\theta}$ takes in $H_{T}^{pre}=[h_0,\cdots,h_{2T-1}]$ as the input matrix, and generates $H^{post}_{T}=TF_{\theta}(H_{T}^{pre})$ as the output. 
from each $t$, suppose the linear extraction map $\mathcal{G}$ , $h^{out}_{t}=H^{post}_{T}[:,2t-1]$, $\hat{Q}_{\theta}(x_t,a_t)=[\mathcal{G} h^{out}_t]_{a_t}$, where $[\mathcal{G} h]_{a}$ be the logit of line $a$. \\
Calculating $y_t=r_t+\max_{a'} \hat{Q}_{\theta}(x_{t+1},a')$,  we then updating the parameter $\theta\in \Theta$ by gradient descent method. 

\textbf{Rollout} 

At online deployment time, we initialize an empty context dataset $D_0$. 
At each timestep $t\in[T]$, given the accumulated context $D_{t-1}$ and the current state $s_t$, we construct the token matrix $H^{pre}_{roll,t} = [h_0,\cdots, h_{2(t-1)-1}, h_{2(t-1)}]$, and generates $H^{post}_{roll,t}=TF_{\hat{\theta}}(H_{roll,t}^{pre})$, from the last output token corresponding to the current state, $h^{out}_{roll,t}=H^{post}_{roll,t}[:,2t-1]$, we compute the action logits via the extraction map $\mathcal{G}$ as $\{[\mathcal{G}_a h_{roll,t}^{out}]\}_{a\in \mathcal{A}}$. 
The agent selects the action $a_t$ according to a greedy policy over these logits $a_t= \arg\max_{a}[\mathcal{G}_a h_{roll,t}^{out}]$. 
After selecting, the agent executes the action, observes the reward $r_t$ and the next state $s_{t+1}$, then construct the new token $h_{2(t-1)+1}$ and append it to the context dataset $D_t$. \\
Given the input token matrix $H^{pre}_{roll,t}$, we construct a Transformer that implements the following steps on the last token. For each token $h^{pre}_{2t-1}$, the Transformer implements the following step-by-step structured transformations:

\[
\begin{bmatrix}
h^{pre,a}_{2t-1} \\
h^{pre,b}_{2t-1} \\
h^{pre,c}_{2t-1}
\end{bmatrix}
\xrightarrow{\text{Step 1 }}
\begin{bmatrix}
h^{pre,a}_{2t-1} \\
0_{\mathcal{A}}\\ r_t \\
pos_{2t-1}
\end{bmatrix}
\xrightarrow{\text{Step 2 }}
\begin{bmatrix}
h^{pre,a}_{2t-1} \\
\max\limits_{a'} \hat{Q}_{\theta}(x_{t+1}, a') \\
r_t \\
pos_{2t-1}
\end{bmatrix}
\]

\[
\xrightarrow{\text{Step 3 }}
\begin{bmatrix}
h^{pre,a}_{2t-1} \\
y_t = r_t + \max\limits_{a'} \hat{Q}_{\theta}(x_{t+1}, a') \\
\star \\
pos_{2t-1}
\end{bmatrix}
\xrightarrow{\text{Step 4 }}
\begin{bmatrix}
h^{post,a}_{2t-1} \\
h^{post,b}_{2t-1} \\
h^{post,c}_{2t-1}
\end{bmatrix}
\]

given the current context $x_t=(D_{t-1},s_t)$, the Transformer $TF_{\theta}(\cdot)$ outputs a token $h^{out}_{2t-1}$, from which the Q-values ${\hat{Q}_{\theta}(x_t,a)}$ are extracted. Each step is defined as below: 

\textbf{Step 1 (Reward Extraction)}  

There exists a attention-only Transformer $\text{TF}_\theta(\cdot)$ that implements Step 1. 

\textbf{Proof of Step 1} 
We prove this step by constructing a Transformer that add $r_t$ from $h_{2t-1}^b$ to $h_{2t+1}^b$. 
We can construct a two-layer attention-only Transformer with $Q^{(1)}_{1,2},K^{(1)}_{1,2},V^{(1)}_{1,2}$, such that for all $i\leq 2t-1$,  
\[
Q_{1}^{(1)}h_i^{(0)}= \begin{bmatrix}1\\i
\end{bmatrix},\quad K^{(1)}_1h_i^{(0)}=\begin{bmatrix}i+2\\-1\end{bmatrix},\quad V_1^{(1)}h^{(0)}_{2t-1}=\begin{bmatrix}0_{|\mathcal{S}|}\\0_{|\mathcal{A}|}\\r_t
\end{bmatrix},\quad V_1^{(1)}h^{(0)}_{2t}=\begin{bmatrix}0_{|\mathcal{S}|}\\0_{|\mathcal{A}|}
\\0\end{bmatrix}, 
\]
and we choose $Q_2^{(1)}=Q^{(1)}_1,V_2^{(1)}=-V^{(1)}_1$, $K_2^{(1)}h^{(0)}_i=\begin{bmatrix}i+1\\-1\end{bmatrix}$, summing up the heads, we obtain the update on a subset of coordinates in $h^{(0),b}_{2t+1}$ as
\[
0_{|\mathcal{S}|+|\mathcal{A}|+1}\rightarrow 0_{|\mathcal{S}|+|\mathcal{A}|+1} + \sum^2_{j=1} \sum_{i=1}^{2t+1} \sigma(\langle Q_j^{(1)}h^{(0)}_{2t+1},K_j^{(1)}h_i^{(0)}\rangle )V_jh_i^{(0)}=\frac{1}{2t+1}\begin{bmatrix} 0_{|\mathcal{S}|}\\0_{|\mathcal{A}|}\\r_t \end{bmatrix}
\]
We then use another attention layer to multiply the updated vectors by a factor of $2t+1$.  

Choosing $Q_1^{(2)}h^{(1)}_i=\sqrt{2t+1}\cdot e_i$, $K_1^{(2)}h^{(1)}_j=\sqrt{2t+1}\cdot e_j$, $V_1^{(2)}h^{(1)}_{2t+1}=\frac{1}{2t+1}\begin{bmatrix}0_S\\0_A\\r_t\end{bmatrix}$, and noting that $\langle Q_1^{(2)}h^{(1)}_i,K_1^{(2)}h^{(1)}_j\rangle =2t+1$, when $j=i$ and otherwise 0. 

\textbf{Step 2 (Future Q-Value Lookup)}  

Attend to the next state token $h^{\text{pre}}_{2t}$, and extract the maximum predicted $\max_{a'} \hat{Q}_{\theta}(x_{t+1},a')$.
There exists a Transformer $\text{TF}_\theta(\cdot)$ that implements Step 2. 

\textbf{Proof of Step 2} 

Given $(\hat{Q}_{\theta})_{T}=0$, we start with constructing an-attention layer, and $\{Q_{jt,s}\}_{s=1}^2$,$\{K_{jt,s}\}_{s=1}^2$,$\{V_{jt,s}\}_{s=1}^2$ such that for all $i\leq 2t-1$ and $j\leq i$, 
\[
Q_{jt,1}^{(1)}h_i^{(0)}=\begin{bmatrix}x_t\\-i\\3T\end{bmatrix},\quad K_{jt,1}^{(1)}h_i^{(0)}=\begin{bmatrix}\hat{Q}_t(\cdot,a_j)\\3T\\j\end{bmatrix},\quad V_{jt,1}^{(1)}h_i^{(0)}=\begin{bmatrix}0\\ie_{jt}\\0\end{bmatrix},\quad Q_{jt,2}^{(1)}h_i^{(0)}=\begin{bmatrix}-x_t\\-i\\3T\end{bmatrix},
\]
$K_{jt,2}^{(1)}=K_{jt,1}^{(1)}$, $V_{jt,2}^{(1)}=-V_{jt,1}^{(1)}$. \\
where $e_{jt}$ is a one-hot vector supported on some entry of $h^c$ . Summing up two heads gives the update for $i\leq 2t-1$, 
\[
0\rightarrow 0+\Big[ \sigma(\langle Q_{jt,1}^{(1)}h_i^{(0)},K_{jt,1}^{(1)}h_i^{(0)}\rangle)-\sigma(\langle Q_{jt,2}^{(1)}h_i^{(0)},K_{jt,2}^{(1)}h_i^{(0)}\rangle) \Big]e_{jt}=\hat{Q}_t(x_t,a_j) e_{jt}
\]
Denote the resulting token vector by $h_i^{(1)}$. 

Next, we construct a MLP layer, 
s.t for any $x\in\mathcal{X}$ on the corresponding coordinates we have 
\[
W^{(2)}_{1}h^{(1)}_i=\begin{bmatrix}\vdots \\ \hat{Q}_t(x,a_1)\\ \hat{Q}_t(x,a_2)-\hat{Q}_t(x,a_1)\\ \vdots \\ \hat{Q}_t(x,a_{|\mathcal{A}|})-\hat{Q}_t(x,a_{|\mathcal{A}|-1})\\ \vdots
\end{bmatrix}
\]
where $a_k$ denotes the $k^{th}$ action, and 
\[
W^{(2)}_2 \sigma(W_1^{(2)}h^{(1)}_i)=\sigma(\hat{Q}_t(x,a_1))+\sum_{k=2}^A \sigma(\hat{Q}_t(x,a_k)-\hat{Q}_t(x,a_{k-1}))=\max_{a\in A}\hat{Q}_t(x,a).
\]

\textbf{Step 3 (TD Target Construction)}  

Combine $r_t$ and $\max_{a'} \hat{Q}(x_{t+1},a')$ to form the TD target:
\[
y_t = r_t + \max_{a'} \hat{Q}(x_{t+1},a').
\]
There exists a Transformer $TF_{\theta}(\cdot)$ that implements Step 3.

\textbf{Proof of Step 3} 

The proof is similar to that of the literature \cite{wang2025transformerslearntemporaldifference}. 
Let
\[
Q_1^{(1)}h^{(0)}_i=K^{(1)}_1h^{(0)}_i=\begin{bmatrix}
0\\1\\0\end{bmatrix},\quad V^{(1)}_1 h^{(0)}_{2t-1}=\begin{bmatrix}
0_{|\mathcal{S}|}\\r_t\\0\end{bmatrix},
\]
\[Q^{(1)}_2 h^{(0)}_i=K^{(1)}_2 h^{(0)}_i=\begin{bmatrix}
0\\0\\1\end{bmatrix},\quad V^{(1)}_2h^{(0)}_{2t}=\begin{bmatrix}
0_{|\mathcal{S}|}\\\max_{a'}\hat{Q}_{\theta}(x_{t+1},a')\\0\end{bmatrix},
\]
then we have that 
\[h^{(1)}_{2t-1}=\sum^2_{j=1}\sum^{2t}_{i=1}\sigma(\langle Q^{(1)}_jh^{(0)}_{2t-1},K^{(1)}_jh^{(0)}_i\rangle)V^{(1)}_j h^{(0)}_i=\begin{bmatrix}0_{|\mathcal{S}|}\\r_t\\\max_{a'}\hat{Q}_{\theta}(x_{t+1},a')\\0
\end{bmatrix},
\]
define a MLP layer with $W^{(2)}_1=\begin{bmatrix}
    0&0&0\\0&1&0\\0&0&1
\end{bmatrix}$, $W^{(2)}_2=\begin{bmatrix}
    0&1&1
\end{bmatrix}$.
We can get that $W^{(2)}_2 \sigma(W_1^{(2)}h^{(1)}_{2t-1})=r_t + \max\limits_{a'} \hat{Q}_{\theta}(x_{t+1}, a') .$

\textbf{Step 4 (Q-Value Update):}

There exists a small MLP head (two-layer feed-forward network) which updates $h_{2t-1}$ to regress the Q-value prediction $\hat{Q}(x_t,a_t)$ towards the TD target $y_t$.

\textbf{Proof of Step 4} 

The proof is related to the literature \cite{vonoswald2023transformerslearnincontextgradient}.
Let 
\[
Q_1^{(1)}h^{(0)}_i=K^{(1)}_1h^{(0)}_i=\begin{bmatrix}0\\1\\0\\0\end{bmatrix}, \quad V^{(1)}_1 h^{(0)}_{2t-1}=\begin{bmatrix}
0_{|\mathcal{S}|}\\\hat{Q}_{\theta}(x_t,a_t)\\0\\0\end{bmatrix},
\]
\[
Q^{(1)}_2 h^{(0)}_i=K^{(1)}_2 h^{(0)}_i=\begin{bmatrix}
0\\0\\1\\0\end{bmatrix},\quad V^{(1)}_2h^{(0)}_{2t-1}=\begin{bmatrix}
0_{|\mathcal{S}|}\\0\\y_t\\0\end{bmatrix}
\]
combining two attention heads, we get that $h^{(1)}_{2t-1}=\begin{bmatrix}0_{|\mathcal{S}|}\\ \hat{Q}_{\theta}(x_t,a_t)\\y_t\\0
\end{bmatrix}.$
define a MLP layer with $W^{(2)}_1=\begin{bmatrix}
    \cdots \\0&1&0&0\\0&0&1&0\\ \cdots
\end{bmatrix}$, $W^{(2)}_2=\begin{bmatrix}
    0&1-\alpha&\alpha&0
\end{bmatrix}$, where $\alpha$ be a coefficient.
We have that $W^{(2)}_2 \sigma(W_1^{(2)}h^{(1)}_{2t-1})=(1-\alpha)\hat{Q}_{\theta}(x_t,a_t)+\alpha y_t .$

\end{proof}

\section{Experiment Details}
\label{sec:exp_details}
\subsection{Computing Resources}
All experiments are conducted on two NVIDIA A100 GPUs (40GB each).

\subsection{LinUCB Algorithm}
\label{ssec:linucb}
Let $T$ denote the time horizon and $\lambda, \alpha > 0$ be input parameters. At each time step $t \in \{1,2,...,T\}$, the LinUCB algorithm operates through the following steps:

\begin{enumerate}
    \item Compute the ridge estimator for the weight vector: \[
    \mathbf{w}^{t}_{\text{ridge}, \lambda} = \arg\min_{\mathbf{w} \in \mathbb{R}^d} \left( \frac{1}{2t} \sum_{j=1}^{t-1} \left( r_j - \langle \mathbf{a}_j, \mathbf{w} \rangle \right)^2 + \frac{\lambda}{2t} \|\mathbf{w}\|_2^2 \right)
    \]
    
    \item For each action $i \in [A]$, compute the upper confidence bound:
    \[
    v_{t,i}^* = \langle \mathbf{a}_{t,i}, \mathbf{w}^{t}_{\text{ridge}, \lambda} \rangle + \alpha \sqrt{\mathbf{a}_{t,i}^{\top} \mathbf{A}_t^{-1} \mathbf{a}_{t,i}},
    \]
    where $\mathbf{A}_t = \lambda \mathbf{I}_d + \sum_{j=1}^{t-1} \mathbf{a}_j \mathbf{a}_j^{\top}$.
    
    \item Select the action $a_{t,j}$ by:
    \[
    j := \arg\max_{i \in [A]} v_{t,i}^*.
    \]
\end{enumerate}

\subsection{Darkroom Environment}
\label{ssec:darkroom_details}

 Darkroom is considered as a complex Markov decision process (MDP) problem and is utilized as a standard benchmark for evaluating in-context learning \cite{laskin2022context, lee2023supervised}. In this experiment, the agent must locate an unknown goal within a 10 × 10 darkroom. The agent receives a reward of 1 only when it reaches the goal. At each step, the agent can choose from five possible actions: move up, down, left, right, or stay still. If the agent is not at the goal, it receives a reward of 0. The horizon for the Darkroom environment is set to 100 steps. We summarize the details as follows:

\paragraph{Pretraining Data Collection} Similar to the stochastic linear bandit problem, we consider two types of policies: A random policy, which selects an action randomly at each position, and an expert policy, which chooses legal actions to avoid crashing into walls and will stay still once the agent receives a reward of 1. To test whether the pretrained model can generalize to \textbf{unseen} RL problems in context, we collect datasets from 80 out of the total 100 goals, reserving the remaining 20 for testing. For each training goal, we run both the random and expert policies, collecting 1k trajectories from each policy. This leads to a total of 80k trajectories for the random policy and 80k trajectories for the expert policy, resulting in 160k context trajectories in the pretraining dataset.

\paragraph{Comparison and Implementation}We evaluate QTPT and SPT under two settings: pretraining on purely random data and on purely expert data. The models are evaluated on an \textbf{unseen} task, where the target goal is not included in the pretraining dataset. The model architecture is identical to that used in the stochastic linear bandit experiments (see Section~\ref{ssec:implement}).

\subsection{Dark Key-to-Door Environment}
\label{ssec:darkk2d}
The Dark Key-to-Door environment extends the complexity of the Darkroom setting by introducing a hierarchical task structure with sparse rewards. In this environment, the agent must sequentially accomplish two objectives: first locate an invisible key to receive a reward of $r = 1$, then find and open an invisible door to receive an additional reward of $r = 1$. This creates a challenging sparse reward scenario where the agent must learn to complete subtasks in the correct order. The environment consists of a 9 × 9 grid, and each episode is limited to 50 steps. At the beginning of each episode, the agent is reset to position $(0,0)$, while the key and door locations are randomly generated across different episodes. Since door and key can be placed in any positions, there are total $81\times81=6561$ distinct tasks. 

\paragraph{Pretraining Data Collection} To capture diverse exploration strategies, we collect our pretraining data using two distinct behavioral policies. The expert policy employs a systematic spiral search pattern, ensuring comprehensive coverage of the environment while efficiently locating the key and door. In contrast, the random policy selects actions uniformly at random, providing diverse but generally suboptimal exploration patterns. To ensure a balanced representation of both behaviors, we collect 80k trajectories from each policy type. The pretraining dataset is partitioned into 5249 tasks, which represents approximately 80\% ($6561\times0.8\approx5249$) of the total tasks. The remaining 1312 tasks are held out for testing and are thus considered \textbf{unseen}.

\paragraph{Comparison and Implementation} We evaluate the performance of our QTPT and SPT models under two distinct pretraining configurations: one trained exclusively on random policy data and the other trained exclusively on expert policy data. The evaluation, conducted on \textbf{unseen tasks}, specifically assesses the models' ability to generalize the underlying hierarchical task structure and adapt to new key-door configurations during in-context learning. To ensure a fair comparison across environments and configurations, the model architecture is kept consistent with that of previous experiments.

\subsection{Miniworld Environment}
\label{ssec:miniworld}
To evaluate QTPT's effectiveness in visual domains, we employ the Miniworld \citep{lee2023supervised} navigation benchmark, which presents a visually grounded navigation task. In this environment, agents must navigate to the correct colored target box among four boxes positioned at the corners of the environment. The agent receives 25 × 25 RGB observations with directional conditioning, providing rich visual information that requires the model to process high-dimensional sensory input. The agent has access to three discrete actions: turn left, turn right, and move forward, which creates a more realistic navigation scenario compared to grid-world abstractions. The agent receives a reward of $r = 1$ only when positioned near the target box, with all other states yielding zero reward. Each episode is constrained to 50 timesteps, requiring efficient navigation strategies.

\paragraph{Pretraining Data Collection} We construct our pretraining dataset using two complementary data collection strategies to capture both optimal and suboptimal navigation behaviors. The expert policy demonstrates intelligent navigation by selecting legal actions that avoid wall collisions and environmental boundaries, and crucially, the agent remains stationary once it successfully reaches the target and receives the reward signal. This policy represents efficient goal-directed behavior in the visual navigation domain. The random policy, conversely, selects actions uniformly at random, generating diverse but often inefficient exploration trajectories that cover various parts of the state space. We collect 24k trajectories from each policy type for training. 

\paragraph{Comparison and Implementation} The evaluation protocol examines QTPT and SPT performance under both pure random data pretraining and pure expert data pretraining conditions. The model architecture incorporates appropriate visual encoding mechanisms while maintaining consistency with the core transformer structure used in other experimental settings.

\subsection{Implementation Detail}
\label{ssec:implement}
The submitted linear-bandit decoder is GPT-2-style \citep{garg2022can}, with 8 layers, 4 attention heads, and embedding dimension $D=256$. Its custom ReLU-attention implementation is not a standard GPT-2 block: it omits the usual attention output projection and uses a $256\to256\to256$ MLP rather than $256\to1024\to256$. The implemented model has 2,858,763 parameters. The original training setup uses batch size 32 and Adam, cosine learning-rate decay with a 2,000-step warmup, peak learning rate $5\times10^{-6}$, minimum $10^{-7}$, and weight decay 0.001. Supplementary experiments state any changes to this setup explicitly.

Each model is pretrained for 50 epochs on the collected datasets. After pretraining, the models are evaluated on a test set to assess their effectiveness in the context of the stochastic linear bandit problem.

\section{Additional Experiments}\label{additional-experi}
\subsection{Ablation Study on the Stochastic Linear Bandit Problem}
\label{sssec:ablation}
We also consider using 'softmax' operation when computing the Q-target in experiments.
To evaluate the robustness of \ours, we conducted ablation experiments examining several key factors: The impact of using either 'softmax' or 'hardmax', the effect of incorporating a double DQN framework for network updates, and the differences between ReLU and softmax activation functions in computing attention scores. Note that our original implementation combines 'softmax', ReLU, and a double DQN framework.
\begin{table*}[ht]
\centering
\caption{Comparison of results for different configurations of \ours.}
\resizebox{\textwidth}{!}{
\begin{tabular}{lcccc}
\toprule
\textbf{Random Data Results} & Default & Double DQN & Softmax + Hardmax & ReLU-Attn + Softmax-Attn \\ 
\midrule
Results & 30.9 & 34.01 & 32.48 & 197.26 \\ 
\midrule
\textbf{Expert Data Results} & Default & Double DQN & Softmax + Hardmax & ReLU-Attn + Softmax-Attn \\ 
\midrule
Results & 39.4 & 45.76 & 53.73 & 203.74 \\ 
\bottomrule
\end{tabular}
}
\end{table*}

In addition, we vary the embedding dimension, number of layers, and number of attention heads to evaluate the model's performance under different configurations. Note that the original configuration consists of 8 layers, 4 attention heads, and an embedding dimension of 256, pretrained on purely random data.

\begin{table}[h!]
\centering
\caption{Results for varying embedding dimensions.}
\begin{tabular}{lccccc}
\toprule
\textbf{Embedding Dimension} & 16 & 32 & 64 & 128 & 256 \\ 
\midrule
Results & 135.7 & 54.0 & 40.5 & 33.3 & 30.9 \\ 
\bottomrule
\end{tabular}
\end{table}

\begin{table}[h!]
\centering
\caption{Results for varying number of attention heads.}
\begin{tabular}{lccccc}
\toprule
\textbf{Number of Heads} & 1 & 2 & 4 & 8 & 16 \\ 
\midrule
Results & 34.2 & 30.2 & 30.9 & 29.3 & 31.4 \\ 
\bottomrule
\end{tabular}
\end{table}

\begin{table}[h!]
\centering
\caption{Results for varying number of layers.}
\begin{tabular}{lcccc}
\toprule
\textbf{Number of Layers} & 1 & 2 & 4 & 8 \\ 
\midrule
Results & 122.3 & 33.6 & 31.0 & 30.9 \\ 
\bottomrule
\end{tabular}
\end{table}

\subsection{Experiments on Non-stationary Environment.}
\label{ssec:non-stat}
We aim to demonstrate the robustness of \ours by evaluating its performance in non-stationary environments.

\paragraph{Experiment Setup:} The pretraining phase remains consistent with the process described in Section 4. Yet, in the test phase, we introduce two different settings: 
\begin{itemize}
    \item \textbf{Stationary:} The parameter $\theta^*$ is sampled from a uniform distribution $[0,1]^d$, which matches the conditions of the pretraining stage;
    \item \textbf{Non-stationary:} The parameter $\theta^*$ is sampled from a standard Gaussian distribution $N(0,1)^d$ and rescaled to lie within the range $[0,1]^d$. The rescaling is performed using the transformation: 
\[
\theta_{\text{scaled}} = \frac{\theta - \theta_{\min}}{\theta_{\max} - \theta_{\min}},
\]
where $\theta_{\min}$ and $\theta_{\max}$ represent the minimum and maximum values of the sampled $\theta$. In our experimental setup, we sample 1000 values from the Gaussian distribution and take the minimum and maximum from these samples for rescaling.

\end{itemize}

Figure \ref{fig:stationary} demonstrates that QTPT achieves superior performance in the non-stationary environment compared to the stationary setting. While SPT maintains relatively consistent performance across both environments, it significantly underperforms compared to Q-learning, especially under non-stationary conditions. These results highlight QTPT's enhanced adaptability to changing environmental dynamics and robustness to distributional shifts in the underlying parameters. Notably, the performance gap between QTPT and SPT widens considerably in the non-stationary setting, with QTPT maintaining strong performance while SPT's effectiveness diminishes.

\begin{figure}[htbp]
    \centering
    \includegraphics[width=0.85\linewidth]{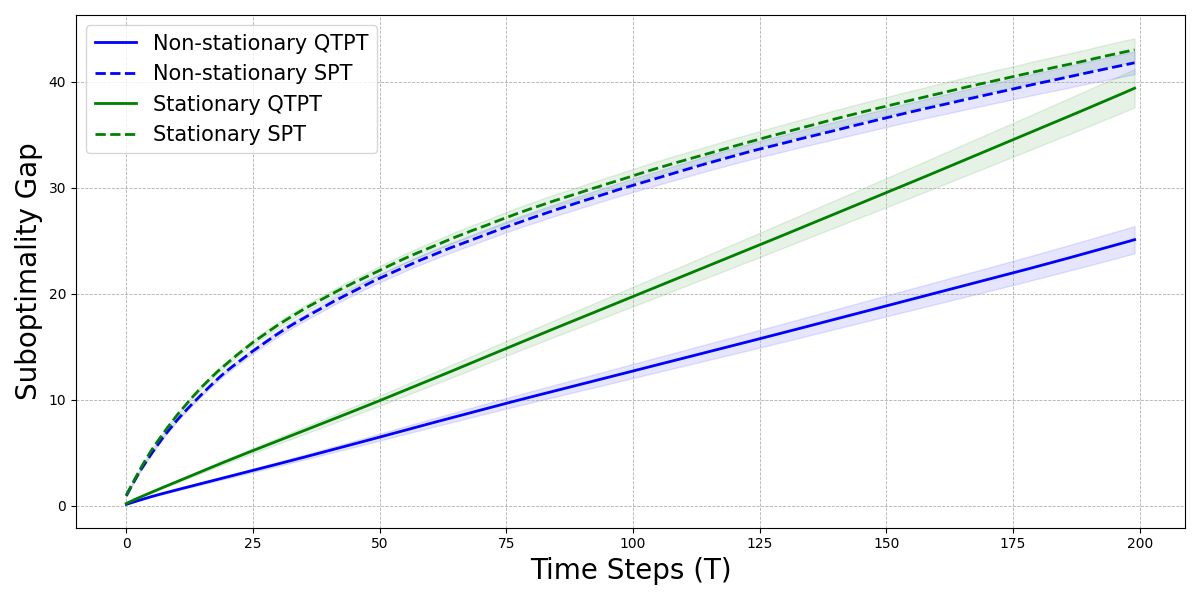}
    \caption{Comparison of Pretrained Transformer Performance Using Q-learning and Supervised Learning in Stationary and Non-stationary Environments. The shaded regions represent the
standard deviation of the suboptimality estimates
based on 1000 simulation runs.}
    \label{fig:stationary}
\end{figure}

\subsection{More Complex Task: Math Reasoning}
\label{ssec:math_reasoning}

\begin{algorithm}
    \caption{QTPT with Policy Optimization (QTPO)}
    \label{algo:bppo}
    \begin{algorithmic}[1]
    \STATE Estimate behavior policy $\hat{\pi}_{\beta}$ by behavior cloning on offline dataset $D$
    \STATE Training the Value function $\hat{V}_{\pi_\beta}$ and advantage function $\hat{A}_{\pi_\beta}$ together by dueling architecture using \ours on offline dataset $D$. 
    \STATE Initialize $k = 0$ and set $\pi_k \leftarrow \hat{\pi}_{\beta}$
    
    \FOR{$i = 0, 1, 2, \cdots, I$}
        \STATE Update the policy $\pi$ by maximizing $L_k(\pi)$
        
        \IF{$J(\pi) > J(\pi_k)$}
            \STATE Set $k = k + 1 \& \pi_k \leftarrow \pi$
        \ENDIF
    \ENDFOR
    \end{algorithmic}
    \end{algorithm}

To demonstrate the broader applicability of \ours, we extended our approach to mathematical reasoning tasks. Each math problem serves as a distinct environment, where generated tokens represent states, the full sequence forms the context-state, and token selection constitutes actions. The reward signal is derived from binary correctness or step-by-step reasoning accuracy. Our method showed consistent performance improvements over supervised learning baselines in these tasks.

A key challenge emerged during preliminary experiments is that when learning a Q-function from offline data alone, and only use the Q-function to generate token, the model frequently produced degenerate outputs (e.g., generating repetitive tokens). We hypothesize that this is due to semantic distortion caused by the Bellman update. Specifically, the output in language models consists of token probabilities, and applying cumulative reward-based reasoning to these probabilities disrupts their linguistic coherence, as it is unconventional to treat them as a direct sum of the reward and the probability of the next token.

To address this challenge, we adopt an actor-critic framework in which the Q-function learned by \ours is not used directly for token generation but to refine the policy. In the offline setting, we modified the Proximal Policy Optimization (PPO) algorithm, drawing inspiration from Behavior Proximal Policy Optimization (BPPO)\cite{Zhuang2023BehaviorPP}, and incorporated an iterative refinement process, which we name \textit{QTPT with Policy Optimization}(QTPO). The details are provided in Algorithm~\ref{algo:bppo}.
where
\begin{equation}
\begin{aligned}
L_k(\pi) = \mathbb{E}_{s \sim \rho_\mathcal{D}(.),\, a \sim \pi_k(.||s)} \Bigg[ 
\min\Bigg(\frac{\pi(a|s)}{\pi_k(a|s)} \hat{A}(s,a), \text{clip}\left(\frac{\pi(a|s)}{\pi_k(a|s)},\, 1-\epsilon,\, 1+\epsilon \right)\hat{A}(s,a) 
\Bigg) \Bigg]
\end{aligned}
\end{equation}
and the objective $J(\pi)$ is the total accuracy on the offline dataset.

\paragraph{Data Preparation} We evaluated our method on the GSM8K dataset \cite{cobbe2021gsm8k}, utilizing reward labels and training data from Math-Shepherd \cite{wang2024mathepherd}. The experimental setup involved two distinct data categories: 
\begin{itemize}[leftmargin=0.2in, nolistsep]
    \item \textbf{Expert Data}, consisting of samples with correct final answers
    \item \textbf{Suboptimal Data}, comprising samples with incorrect answers.
\end{itemize}

Each category contained 50,000 samples. 
\paragraph{Implementation and Comparison}
We initialized both the policy and Q-function models using the Qwen-2.5-1.5B-instruct model. These models were then trained using Low-Rank Adaptation (LoRA) with the following hyperparameters: rank=8, lora\_alpha=32, and dropout=0.1.
We evaluated performance on the \textbf{unseen} GSM8k test set. This evaluation compared QTPT against direct behavior cloning, using models trained on various mixtures of expert and imperfect data.

\begin{figure}[htbp]
    \centering
    \includegraphics[width=0.72\linewidth]{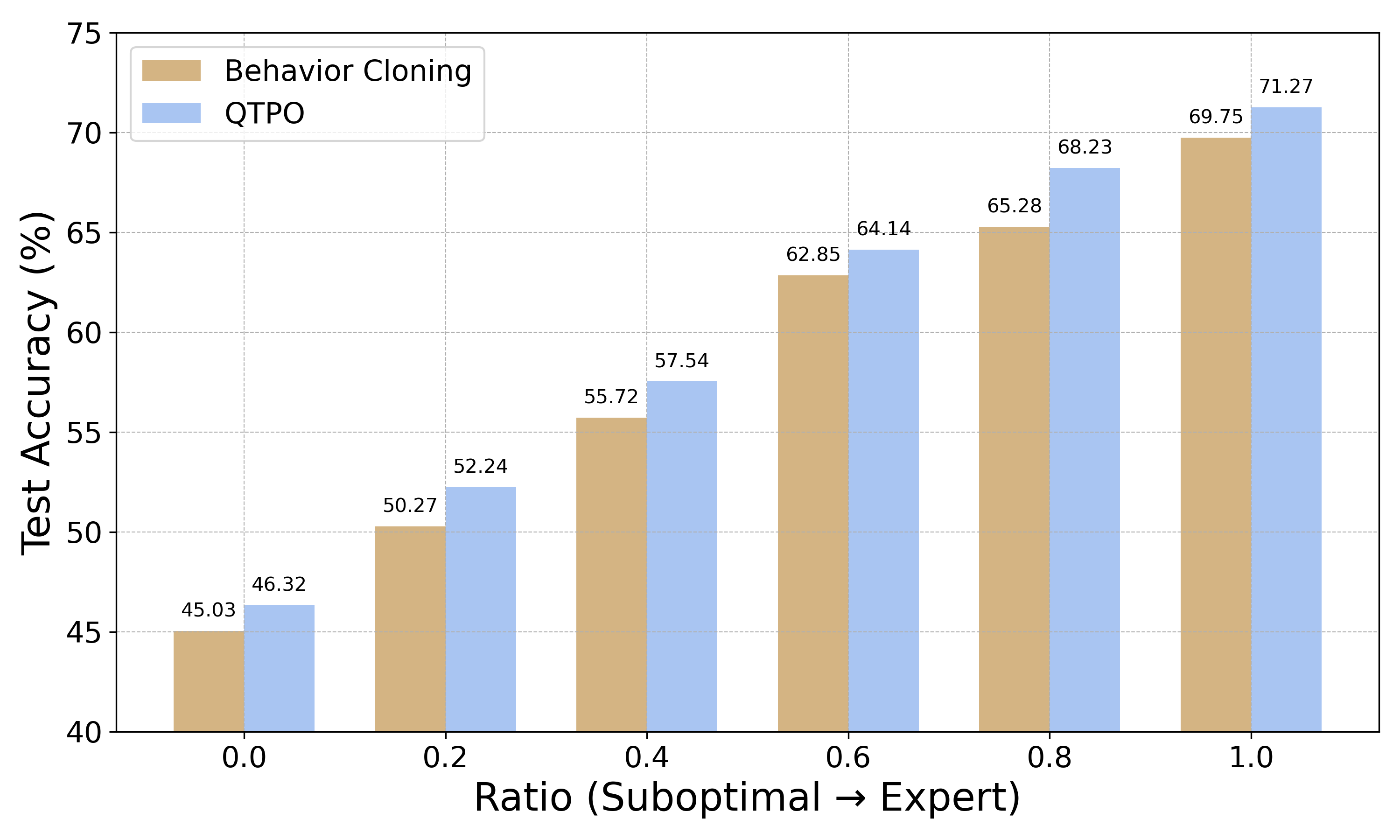}
    \caption{Accuracy (\%) on the GSM8K test set across different mixtures of suboptimal and expert data.}
    \label{fig:gsm8k}
\end{figure}

\paragraph{Results}The results of these comparisons are presented in the figure ~\ref{fig:gsm8k}. QTPT consistently outperformed behavior cloning across all settings, achieving an average accuracy improvement of \textbf{1.81\%}. This consistent enhancement is particularly noteworthy as our method relies exclusively on a static offline dataset.

\subsection{D4RL Extension and Controlled Ablations}
\label{app:post_submission}

The D4RL extension uses Franka Kitchen and AntMaze \citep{fu2020d4rl}. To keep the main text and appendix consistent, the available active-stage summaries are $13.23\pm4.03$ for \texttt{kitchen-partial-v0}, $14.62\pm3.36$ for \texttt{kitchen-mixed-v0}, and $35.94\pm13.52$ for \texttt{antmaze-umaze-v2}. The first two are normalized scores; the last is success percentage. Active-stage means condition on nonzero periodic evaluation checkpoints and therefore exclude failures. These are descriptive 3--4-seed matched-Transformer results, not standard final-policy benchmark scores or evidence of statistically significant superiority. No additional D4RL run is included in this revision.

Table~\ref{tab:post_submission_ablation} compares the submitted TD, Monte Carlo, and no-context variants. Its intervals are over evaluation episodes, not independent training seeds. The context and OOD diagnostics below address distinct mechanisms under separately stated protocols; they do not turn episode-level uncertainty into seed-level significance.

\subsection{Comparison with Other Baselines}
\label{app:baseline_compare}

To provide a comprehensive and principled evaluation, we compare QTPT against representative methods from two major paradigms of offline decision learning, distinguished by their learning mechanisms rather than surface algorithmic forms.

\textbf{(1) Bellman-based value-function-centric offline RL.}  
This class is represented by Conservative Q-Learning (CQL, \cite{kumar2020conservativeqlearningofflinereinforcement}), which explicitly relies on Bellman backups and value function optimization with conservative regularization to address distributional shift and out-of-distribution actions. Policy extraction is performed through value maximization over the learned Q-function.

\textbf{(2) Value-guided policy learning.}  
IQL \citep{kostrikov2021offlinereinforcementlearningimplicit} trains its critic with an in-sample value backup and extracts a policy by advantage-weighted imitation. Q-SFT \citep{hong2024qsftqlearninglanguagemodels} and QDT \citep{yamagata2023q} provide value-guided supervision for sequence-model policies. They differ in their critic construction and targets; the presence of a supervised policy loss does not make IQL Bellman-free.

The submitted linear-bandit comparison uses identical offline data and no test-time parameter updates. The Q-guided policy variants share a pretrained critic in this implementation. Consequently, these labels identify matched value-guided variants, not unchanged implementations of every original algorithm; in particular, the shared-critic IQL-style policy is not a full independent IQL reproduction. We retain the figure as that matched comparison, not as a universal ranking of the original offline-RL algorithms.

As shown in Figure \ref{fig:somebaseline}, QTPT achieves lower cumulative regret than the matched baselines. 
On the expert dataset, baselines like QDT and CQL achieve moderate performance but still incur higher cumulative regret compared to QTPT. The performance gap becomes more pronounced on the random dataset. In this setting, QTPT maintains a low-regret profile, demonstrating robustness to suboptimal data. 

The figure is descriptive evidence about the matched implementations. In the linear bandit, past action--reward observations provide information about the hidden task parameter, so it does test task inference; it does not isolate every aspect of that inference. Appendix~\ref{app:context_diagnostics} separately tests the role of task-relevant history.

\begin{figure}[htbp]
    \centering
    \includegraphics[width=0.6\linewidth]{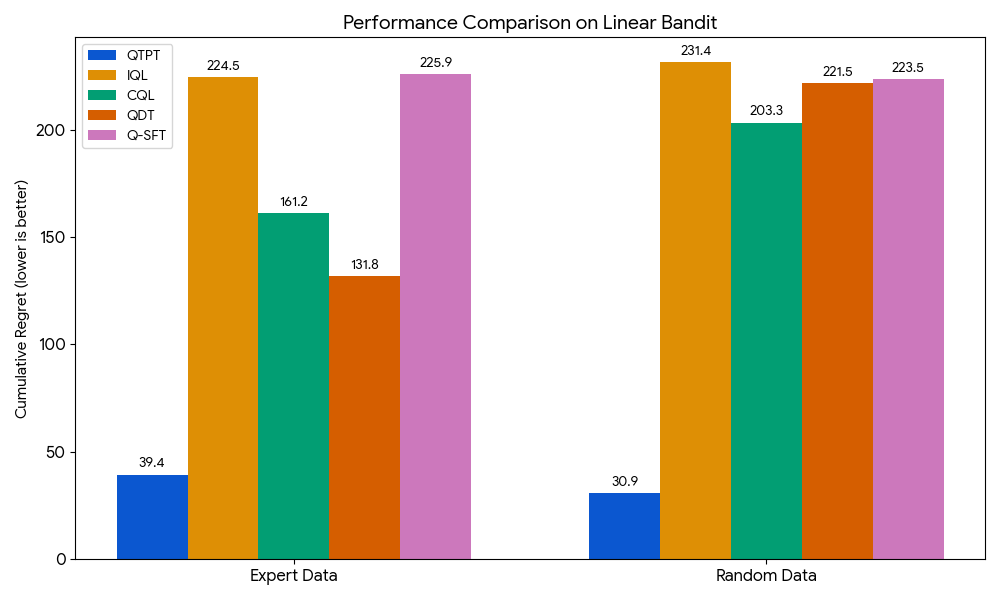}
    \caption{Performance comparison against strong baselines on the Stochastic Linear Bandit task}
    \label{fig:somebaseline}
\end{figure}


\input{sections/supplementary_rebuttal}

%% file: sections/supplementary_rebuttal.tex
\section{Supplementary Evaluation and Clarifications}
\label{app:supplementary}

\subsection{Target Construction, Stochasticity, and Cost}
\label{app:target_details}

\paragraph{Idealized objective and practical backup.}
The formal analysis uses the undiscounted expected Bellman operator. The submitted implementation uses online values for action weights and a separate target network for evaluation \citep{rev_vanhasselt2016double}:
\[
p_\theta(a\mid x_{t+1})=
\frac{\exp(Q_\theta(x_{t+1},a)/\tau)}{\sum_b\exp(Q_\theta(x_{t+1},b)/\tau)},
\qquad
y_t=r_t+\gamma\sum_a p_\theta(a\mid x_{t+1})Q_{\bar\theta}(x_{t+1},a).
\]
Here $x_t=D_{t-1}\cup\{s_t\}$, $\tau=1$, and target parameters are copied from the online model every 100 optimizer steps. The support-controlled diagnostic below uses $\gamma=0.99$. Gradients do not pass through $y_t$. This Double-target/softmax combination stabilizes target construction but supplies no conservative lower-bound guarantee for unsupported actions.

\paragraph{Stochastic targets.}
For fixed next-step values, write $Y_t=R_t+\max_{a'}Q_{t+1}(X_{t+1},a')$ and $m_t(x,a)=\mathbb E[Y_t\mid X_t=x,A_t=a]$. Then
\[
\mathbb E[(f(X_t,A_t)-Y_t)^2]
=\mathbb E[(f(X_t,A_t)-m_t(X_t,A_t))^2]
+\mathbb E[\operatorname{Var}(Y_t\mid X_t,A_t)].
\]
For this fixed-target regression, the last term does not depend on $f$. Stochastic rewards and transitions change target variance, not the population regression minimizer. The identity does not assert that the variance is constant when the target network changes between updates.

\paragraph{Memory and computation.}
Time-indexed Q-functions share one Transformer. Targets are constructed per batch, not stored for every offline trajectory. A model with $P$ parameters requires one additional $O(P)$ target copy, one no-gradient target forward pass per batch, and a hard copy every 100 updates. The target carries neither gradients nor optimizer states and is absent at deployment. For the custom 2,858,763-parameter bandit model, the FP32 copy is approximately 10.91 MiB. Reduced-precision storage and sharding are possible scaling options, not additional experiments reported here.

\paragraph{Attention versus backup sensitivity.}
For fixed $q\in\mathbb R^{|\mathcal A|}$ and $p=\operatorname{softmax}(q/\tau)$,
\[
0\leq\max_a q_a-\sum_a p_aq_a\leq\tau\log|\mathcal A|.
\]
For fixed $p$, replacing online by target-network evaluation changes the weighted value by at most $\|Q_\theta(x,\cdot)-Q_{\bar\theta}(x,\cdot)\|_\infty$. These are one-step target comparisons, not guarantees about a complete nonlinear training trajectory. Attention normalization instead changes the history representation. For identical positive attention logits $M$ and value vectors $v$, length-normalized ReLU aggregation gives $Mv$, whereas softmax gives $v$. This illustrates a scale difference not controlled by the backup bound. It offers a plausible representation-level explanation for the original attention ablation, without proving a unique cause of its performance collapse. Attention choices should be validated separately from target-construction hyperparameters.

\subsection{Unsupported-Action Overestimation}
\label{app:ood}

\paragraph{Controlled support and training protocol.}
We use a stochastic linear bandit with ten five-dimensional action vectors and $T=200$. Rewards are $r_t=a_t^\top w^\star+\epsilon_t$, $\epsilon_t\sim\mathcal N(0,1.5^2)$. For each task, seven actions form $S$ and the other three form $U$. Behavior samples uniformly from $S$; every supported action appears at least once, and no action in $U$ is executed in the offline trajectory. All ten vectors remain visible. The split is randomized across tasks, testing behavior-support OOD rather than a fixed action-index rule.

Each paired seed uses 10,000 offline tasks, an 8,000/2,000 train/validation split, and 1,000 disjoint test tasks. The model has 8 layers, width 256, 4 heads, dropout 0.2, batch size 128, learning rate $5\times10^{-6}$, and weight decay $10^{-3}$. Training is capped at 100 epochs, with patience 20. The minimum held-out validation objective selects the checkpoint, never test performance.

The $2\times2$ design compares QTPT with Double-target/softmax backup against single-network hard-max QTPT, with or without the discrete CQL(H) penalty \citep{kumar2020conservativeqlearningofflinereinforcement}:
\[
L_{\mathrm{TD}}+\alpha L_{\mathrm{CQL}},\qquad
L_{\mathrm{CQL}}=\mathbb E\left[\log\sum_a \exp Q_\theta(x_t,a)-Q_\theta(x_t,a_t^{\mathrm{data}})\right],
\qquad \alpha=0.03.
\]
QTPT + CQL penalty adds this term to QTPT; it is not an independent MLP CQL baseline. All four cells use paired data and evaluation tasks.

\paragraph{Exact terminal calibration.}
After $T-1$ support-only observations, no future reward remains beyond the final action, so the task-specific oracle value is exactly $\mu(a)=a^\top w^\star$. Define
\[
\Delta_{\mathrm{OOD}}=
\frac1{|U|}\sum_{a\in U}(\hat Q_T(x_T,a)-\mu(a))
-\frac1{|S|}\sum_{a\in S}(\hat Q_T(x_T,a)-\mu(a)),
\]
\[
G_{\mathrm{OOD}}=
\left[\max_{a\in U}\hat Q_T(x_T,a)-\max_{a\in S}\hat Q_T(x_T,a)\right]
-\left[\max_{a\in U}\mu(a)-\max_{a\in S}\mu(a)\right].
\]
For greedy action $a_g$, harmful selection means $a_g\in U$ and $\mu(a_g)<\max_{a\in S}\mu(a)$. OOD-attributable regret is this event's indicator times $\max_{a\in S}\mu(a)-\mu(a_g)$, averaged over all tasks. Separately, full-rollout regret is $\sum_{t=1}^{T}[\max_a\mu(a)-\mu(a_t)]$ when all actions are available from step one. Correctly choosing an optimal unsupported action is not a failure.

\begin{table}[htbp]
\centering\small
\setlength{\tabcolsep}{3pt}
\caption{Support-controlled diagnostic. Mean $\pm$ standard deviation across five paired training/data seeds. Lower is better; harmful rate is a fraction.}
\label{tab:ood}
\begin{tabular}{lrrrr}
\toprule
Method & $G_{\mathrm{OOD}}\downarrow$ & Harmful rate $\downarrow$ & OOD regret $\downarrow$ & Rollout regret $\downarrow$\\
\midrule
QTPT & $0.2569\pm0.1021$ & $0.2064\pm0.0362$ & $0.1828\pm0.0754$ & $153.15\pm58.99$\\
Naive hard-max QTPT & $0.3351\pm0.0387$ & $0.2588\pm0.0090$ & $0.2856\pm0.0157$ & $211.67\pm7.81$\\
QTPT + CQL penalty & $0.2582\pm0.1021$ & $0.2144\pm0.0390$ & $0.1903\pm0.0675$ & $139.86\pm55.74$\\
Naive hard-max + CQL penalty & $0.3330\pm0.0386$ & $0.2514\pm0.0084$ & $0.2840\pm0.0176$ & $205.40\pm10.36$\\
\bottomrule
\end{tabular}
\end{table}

\paragraph{Interpretation.}
Task averages are computed within each seed before seed-level means, sample standard deviations, and Student-$t$ intervals. The offset $\Delta_{\mathrm{OOD}}$ is near zero (between $-0.0066$ and $-0.0033$), but QTPT's max-gap, harmful selection, and OOD regret have 95\% intervals strictly above zero. This is policy-relevant max-selection error, not uniform inflation of every unsupported value.

Naive hard-max minus QTPT increases harmful selection by 0.0524 and OOD regret by 0.1028; both paired 95\% intervals are positive. The combined stabilization reduces harm, without isolating each stabilizer's contribution. In either backup branch, paired intervals for CQL-induced changes in all four metrics contain zero. This does not establish reliable incremental improvement, nor prove equivalence.

Some unsupported actions are genuinely optimal and inferable through shared task structure. Support is context- and task-conditional. CQL is evaluated per context, not as a global blacklist, but its all-action penalty does not explicitly distinguish harmful unsupported actions from beneficial generalization. This is a plausible explanation, not a demonstrated causal mechanism.

\begin{table}[htbp]
\centering
\caption{Darkroom diagnostic: five-seed held-out reward, mean $\pm$ standard deviation, validation-best checkpoints. This protocol is separate from Table~\ref{tab:post_submission_ablation}.}
\label{tab:darkroom_cql}
\begin{tabular}{crr}
\toprule
Expert mix & QTPT reward $\uparrow$ & QTPT + CQL penalty reward $\uparrow$\\
\midrule
0.0 & $8.3053\pm8.7520$ & $9.3333\pm8.7971$\\
0.5 & $15.7753\pm10.7997$ & $14.0393\pm7.8813$\\
1.0 & $6.2247\pm1.9315$ & $8.9653\pm4.2767$\\
\bottomrule
\end{tabular}
\end{table}
These runs use $d=10$, $T=100$, five actions, 10,000 training and 100 test environments, and 4-layer, width-128, 4-head models, with 20,000 updates, batch size 128, and learning rate $3\times10^{-4}$. The mean difference changes direction across regimes and variability is substantial. These are performance observations, not Q-calibration measurements. Understanding conservatism in context-conditioned generalization remains a promising research direction.

\subsection{Meta-RL Comparisons}
\label{app:meta_rl}

RL$^2$ \citep{rev_duan2016rl2}, PEARL \citep{rev_rakelly2019pearl}, and QTPT all adapt behavior using task evidence, but differ in pretraining data acquisition and representation. Matching a data budget does not make online collection identical to offline learning.

\paragraph{RL$^2$: transition-budget comparison.}
A five-arm, ten-pull Bernoulli bandit has unknown arm probabilities per task. QTPT receives a fixed random-behavior corpus; RL$^2$ collects fresh on-policy interactions. One pull is one transition, so 1M/5M transitions correspond to 100,000/500,000 trials. Evaluation uses 1,000 common tasks and common outcome tapes. Pseudo-regret is $\sum_{t=1}^{10}(p^\star-p_{a_t})$.
Bayes-normalized progress is $100(R_{\mathrm{random}}-R)/(R_{\mathrm{random}}-R_{\mathrm{Bayes}})$, where $R_{\mathrm{random}}=3.333333$ and $R_{\mathrm{Bayes}}=1.673345$: random is 0\% and the finite-horizon Bayes-optimal reference is 100\%.

\begin{table}[htbp]
\centering
\caption{Bernoulli bandit, mean $\pm$ standard deviation over three independent training seeds.}
\label{tab:rl2}
\begin{tabular}{llrr}
\toprule
Method / acquisition & Transitions & Pseudo-regret $\downarrow$ & Bayes progress $\uparrow$\\
\midrule
QTPT / fixed random corpus & 1M & $1.915\pm0.085$ & $85.4\%\pm5.1\%$\\
RL$^2$ / on-policy trials & 1M & $2.356\pm0.112$ & $58.9\%\pm6.7\%$\\
QTPT / fixed random corpus & 5M & $1.784\pm0.041$ & $93.3\%\pm2.5\%$\\
RL$^2$ / on-policy trials & 5M & $1.780\pm0.056$ & $93.6\%\pm3.4\%$\\
\bottomrule
\end{tabular}
\end{table}
QTPT has lower mean pseudo-regret at 1M transitions; at 5M the means are almost identical. This supports a budget-dependent comparison, not universal superiority of offline acquisition.

\paragraph{PEARL: common fixed-offline replay.}
A separate continuous ten-pull meta-bandit has latent target $M$ and expected reward $1-(a-M)^2$. Both methods receive the same 80,000 uniform-random transitions and 10,000 updates. PEARL retains probabilistic context inference, posterior sampling, and SAC losses, but replaces active collection with this replay corpus. Evaluation uses seven training seeds, 1,000 common tasks, and two common noise tapes. Checkpoints use the fixed training endpoint, not test-set selection.

\begin{table}[htbp]
\centering
\caption{Fixed-offline continuous meta-bandit. Mean $\pm$ standard deviation over seven training seeds; pseudo-regret is $\sum_{t=1}^{10}(a_t-M)^2$.}
\label{tab:pearl}
\begin{tabular}{lrr}
\toprule
Method & Pseudo-regret $\downarrow$ & Mean training time (s)\\
\midrule
QTPT & $0.969\pm0.140$ & 842.7\\
PEARL (fixed-offline adaptation) & $2.472\pm0.074$ & 1062.7\\
\bottomrule
\end{tabular}
\end{table}
The paired difference is $-1.503\pm0.173$; an exploratory bootstrap 95\% interval is $[-1.621,-1.383]$, and the exact two-sided sign test gives $p=0.0156$. This supports an advantage under this fixed weak-corpus protocol, not reproduction of PEARL's original benchmark. The two meta-RL tables have different task/action spaces and must not be numerically pooled.

\subsection{Backbone and Mixture Robustness}
\label{app:backbone}

This supplementary bandit study uses small randomly initialized GPT-2-, Llama-, and Qwen2-style decoders, comparing QTPT and SPT within each matched backbone rather than fine-tuning pretrained language models. Its independent seeds and training budget differ from the original bandit setting, so the absolute GPT-2 values should not be equated with Figure~\ref{fig:mix}.

The supplementary configuration uses 3 layers, width 128, 4 heads, dropout 0.1, 10,000 training trajectories, 1,000 validation trajectories, and 1,000 test tasks. Training uses 3,000 updates, batch size 64, AdamW with learning rate $3\times10^{-4}$ and weight decay $10^{-4}$. Validation losses are logged, but the table evaluates the fixed 3,000-update endpoint. Its target-network update is Polyak averaging with coefficient 0.01; this is a matched-backbone diagnostic, not an unchanged rerun of the submitted 8-layer setup.

\begin{table}[htbp]
\centering
\caption{Cumulative gap at $T=200$, mean $\pm$ standard deviation over five training seeds. Mix is expert-trajectory fraction. Lower is better.}
\label{tab:backbone}
\begin{tabular}{lcrr}
\toprule
Backbone & Expert mix & QTPT gap $\downarrow$ & SPT gap $\downarrow$\\
\midrule
GPT-2 & 0.0 & $33.52\pm3.85$ & $222.20\pm2.35$\\
GPT-2 & 0.5 & $34.42\pm2.10$ & $133.41\pm13.97$\\
GPT-2 & 1.0 & $35.53\pm3.19$ & $46.92\pm1.36$\\
Llama & 0.0 & $33.18\pm1.44$ & $221.93\pm2.19$\\
Llama & 0.5 & $40.59\pm2.49$ & $141.29\pm11.15$\\
Llama & 1.0 & $38.11\pm2.07$ & $39.88\pm0.93$\\
Qwen2 & 0.0 & $33.88\pm1.34$ & $222.09\pm1.94$\\
Qwen2 & 0.5 & $39.35\pm0.70$ & $141.08\pm11.75$\\
Qwen2 & 1.0 & $36.54\pm2.17$ & $39.98\pm1.25$\\
\bottomrule
\end{tabular}
\end{table}
QTPT has lower mean gap in all nine comparisons, with a large random-data margin and a narrower expert-data margin. Similar random-data SPT means are consistent with imitation of nearly uniform behavior, but numerical proximity alone does not prove exact convergence to a random policy. These tests cover three decoder choices, not all model scales.

\begin{table}[htbp]
\centering\small
\caption{Values used for Figure~\ref{fig:mix}; mean $\pm$ standard deviation over five training seeds per mixture.}
\label{tab:mix_sd}
\begin{tabular}{crr}
\toprule
Expert mix & QTPT gap $\downarrow$ & SPT gap $\downarrow$\\
\midrule
0.00 & $31.97\pm1.45$ & $220.85\pm1.24$\\
0.10 & $31.58\pm0.41$ & $210.40\pm1.24$\\
0.20 & $32.57\pm0.66$ & $196.89\pm11.14$\\
0.25 & $33.92\pm1.25$ & $173.79\pm9.33$\\
0.50 & $36.24\pm1.57$ & $139.98\pm11.59$\\
0.75 & $36.43\pm1.75$ & $78.30\pm2.56$\\
0.80 & $35.64\pm0.90$ & $77.94\pm6.66$\\
0.90 & $37.93\pm2.44$ & $59.95\pm11.76$\\
1.00 & $38.14\pm2.02$ & $46.32\pm1.83$\\
\bottomrule
\end{tabular}
\end{table}

\subsection{Task-Coherent Context and Extra Episodes}
\label{app:context_diagnostics}

These synthetic diagnostics isolate context information rather than rerun the full sequential Bellman benchmark. A hidden goal $g\in[-0.9,0.9]$ determines reward $r(a,g)=1-(a-g)^2+\epsilon$, with noise standard deviation 0.05 and eleven grid actions in $[-1,1]$. Behavior selects the nearest-goal action with probability 0.55 and a uniform random action otherwise. Each query has two observed transitions; support-augmented inputs add 24 transitions. The true goal is visible only in the oracle condition.

Both studies use five training seeds, 4,000 updates, batch size 256, learning rate $3\times10^{-4}$, and a two-layer, width-96, four-head Transformer. The objective regresses queried-action reward; at this terminal query the value equals expected immediate reward. All context observations precede the query, but the diagnostic encoder uses bidirectional attention over that observed context, not the submitted causal decoder. Training resamples tasks. Every 500 updates, checkpoints are evaluated on 512 goal-grid points with eight context resamples, selecting minimum diagnostic regret. These exploratory results use the same evaluation for selection and reporting, not a separate final hold-out.

\begin{table}[htbp]
\centering
\caption{Diagnostic A: extra-context source, five-seed mean $\pm$ standard deviation. Regret is true reward loss against the best grid action; success is the fraction choosing it.}
\label{tab:context_source}
\begin{tabular}{lrrr}
\toprule
Input protocol & Regret $\downarrow$ & Success $\uparrow$ & Action flip rate\\
\midrule
Current-episode history & $0.0033\pm0.0003$ & $0.8808\pm0.0097$ & $0.2036\pm0.0120$\\
History + unrelated-task support & $0.0037\pm0.0002$ & $0.8694\pm0.0096$ & $0.2220\pm0.0074$\\
History + same-task support & $0.0002\pm0.0000$ & $0.9638\pm0.0069$ & $0.0475\pm0.0099$\\
\bottomrule
\end{tabular}
\end{table}

\begin{table}[htbp]
\centering
\caption{Diagnostic B: task relevance versus token count, five-seed mean $\pm$ standard deviation. Displayed $0.0000$ values are rounded.}
\label{tab:context_relevance}
\begin{tabular}{lrrr}
\toprule
Input protocol & Regret $\downarrow$ & Success $\uparrow$ & Action flip rate\\
\midrule
No history & $0.2386\pm0.0000$ & $0.1172\pm0.0000$ & $0.0000\pm0.0000$\\
Current-episode history & $0.0033\pm0.0003$ & $0.8847\pm0.0100$ & $0.1978\pm0.0203$\\
Unrelated history (same length) & $0.2386\pm0.0000$ & $0.1172\pm0.0000$ & $0.0000\pm0.0000$\\
History + same-task support & $0.0002\pm0.0001$ & $0.9536\pm0.0112$ & $0.0502\pm0.0052$\\
Oracle goal token & $0.0001\pm0.0001$ & $0.9547\pm0.0110$ & $0.0000\pm0.0000$\\
\bottomrule
\end{tabular}
\end{table}
Same-task support improves the diagnostic, whereas unrelated support does not. Removing history or replacing it with equally long unrelated history loses task information. Action flip rate measures sensitivity to context resampling; a constant bad policy also has zero flips, so it is not unconditionally lower-is-better. The results support the task-coherence explanation, not a theorem that arbitrary extra context must hurt or a long-horizon policy advantage over every baseline.

\paragraph{Reporting conventions.}
Each supplementary table specifies seed count and uncertainty unit. Seed standard deviations, episode confidence intervals, paired seed tests, and conditional checkpoint aggregates are distinct. Different selection rules and information protocols are not pooled.

%% file: checklist.tex
\section*{NeurIPS Paper Checklist}

\begin{enumerate}

\item {\bf Claims}
    \item[] Question: Do the main claims made in the abstract and introduction accurately reflect the paper's contributions and scope?
    \item[] Answer: \answerYes{}.
    \item[] Justification: The abstract and introduction state the main theoretical and empirical claims, including the scope of the D4RL extension and controlled ablations.

\item {\bf Limitations}
    \item[] Question: Does the paper discuss the limitations of the work performed by the authors?
    \item[] Answer: \answerYes{}.
    \item[] Justification: The conclusion discusses remaining limitations and future work, including broader benchmarks, alternative Bellman backup designs, OOD overestimation, and distribution shift.

\item {\bf Theory assumptions and proofs}
    \item[] Question: For each theoretical result, does the paper provide the full set of assumptions and a complete (and correct) proof?
    \item[] Answer: \answerYes{}.
    \item[] Justification: The theoretical assumptions and statements are provided in the analysis section, with additional derivations and proofs included in the appendix.

    \item {\bf Experimental result reproducibility}
    \item[] Question: Does the paper fully disclose all the information needed to reproduce the main experimental results of the paper to the extent that it affects the main claims and/or conclusions of the paper (regardless of whether the code and data are provided or not)?
    \item[] Answer: \answerYes{}.
    \item[] Justification: The paper describes the benchmark tasks, data-generation protocols, model architecture, optimizer, and major hyperparameters in the main text and appendix.

\item {\bf Open access to data and code}
    \item[] Question: Does the paper provide open access to the data and code, with sufficient instructions to faithfully reproduce the main experimental results, as described in supplemental material?
    \item[] Answer: \answerNo{}.
    \item[] Justification: The experiments use standard public benchmarks and described data-generation procedures, but an anonymized public code release is not provided at submission time.

\item {\bf Experimental setting/details}
    \item[] Question: Does the paper specify all the training and test details (e.g., data splits, hyperparameters, how they were chosen, type of optimizer) necessary to understand the results?
    \item[] Answer: \answerYes{}.
    \item[] Justification: The appendix specifies the GPT-2-style Transformer configuration, optimizer, learning-rate schedule, batch size, training epochs, and task-specific experimental details.

\item {\bf Experiment statistical significance}
    \item[] Question: Does the paper report error bars suitably and correctly defined or other appropriate information about the statistical significance of the experiments?
    \item[] Answer: \answerNo{}.
    \item[] Justification: Simulation-based results include standard-deviation bands where applicable, but some larger benchmark summaries, including the D4RL extension, are reported as point estimates due to computational cost.

\item {\bf Experiments compute resources}
    \item[] Question: For each experiment, does the paper provide sufficient information on the computer resources (type of compute workers, memory, time of execution) needed to reproduce the experiments?
    \item[] Answer: \answerNo{}.
    \item[] Justification: The paper provides model and optimization details, but it does not yet fully specify hardware type, memory, runtime, or total compute for each experiment.
    
\item {\bf Code of ethics}
    \item[] Question: Does the research conducted in the paper conform, in every respect, with the NeurIPS Code of Ethics \url{https://neurips.cc/public/EthicsGuidelines}?
    \item[] Answer: \answerYes{}.
    \item[] Justification: The work uses standard reinforcement-learning and reasoning benchmarks and does not involve human subjects, private data, or deceptive data collection.

\item {\bf Broader impacts}
    \item[] Question: Does the paper discuss both potential positive societal impacts and negative societal impacts of the work performed?
    \item[] Answer: \answerNo{}.
    \item[] Justification: The paper includes a broader-impact statement, but it does not provide a detailed separate discussion of both positive and negative societal impacts because the work is primarily foundational.
    
\item {\bf Safeguards}
    \item[] Question: Does the paper describe safeguards that have been put in place for responsible release of data or models that have a high risk for misuse (e.g., pre-trained language models, image generators, or scraped datasets)?
    \item[] Answer: \answerNA{}.
    \item[] Justification: The paper does not release high-risk pretrained generative models, scraped datasets, or other assets requiring special misuse safeguards.

\item {\bf Licenses for existing assets}
    \item[] Question: Are the creators or original owners of assets (e.g., code, data, models), used in the paper, properly credited and are the license and terms of use explicitly mentioned and properly respected?
    \item[] Answer: \answerNo{}.
    \item[] Justification: The paper cites the existing benchmarks and models used in the experiments, but it does not explicitly list all license terms in the manuscript.

\item {\bf New assets}
    \item[] Question: Are new assets introduced in the paper well documented and is the documentation provided alongside the assets?
    \item[] Answer: \answerNA{}.
    \item[] Justification: The paper proposes a method and does not introduce a new dataset or benchmark asset.

\item {\bf Crowdsourcing and research with human subjects}
    \item[] Question: For crowdsourcing experiments and research with human subjects, does the paper include the full text of instructions given to participants and screenshots, if applicable, as well as details about compensation (if any)? 
    \item[] Answer: \answerNA{}.
    \item[] Justification: The work does not involve crowdsourcing or human-subject experiments.

\item {\bf Institutional review board (IRB) approvals or equivalent for research with human subjects}
    \item[] Question: Does the paper describe potential risks incurred by study participants, whether such risks were disclosed to the subjects, and whether Institutional Review Board (IRB) approvals (or an equivalent approval/review based on the requirements of your country or institution) were obtained?
    \item[] Answer: \answerNA{}.
    \item[] Justification: The work does not involve human-subject research, so IRB approval or equivalent review is not applicable.

\item {\bf Declaration of LLM usage}
    \item[] Question: Does the paper describe the usage of LLMs if it is an important, original, or non-standard component of the core methods in this research? Note that if the LLM is used only for writing, editing, or formatting purposes and does \emph{not} impact the core methodology, scientific rigor, or originality of the research, declaration is not required.
    \item[] Answer: \answerYes{}.
    \item[] Justification: The appendix describes the GSM8K extension using Qwen-2.5-1.5B-Instruct with LoRA; this is an experimental extension rather than the core theoretical method.

\end{enumerate}